\documentclass[10pt,twocolumn,letterpaper]{article}
\ifdefined\pdfminorversion\pdfminorversion=7\fi
\ifdefined\pdfsuppresswarningpagegroup\pdfsuppresswarningpagegroup=1\fi
\PassOptionsToPackage{table}{xcolor}
\usepackage[pagenumbers]{cvpr}
\def\confName{CVPR}
\def\confYear{2027}

\usepackage{amsmath,amsfonts,bm}

\def\eqref#1{equation~\ref{#1}}

\def\1{\bm{1}}

\DeclareMathAlphabet{\mathsfit}{\encodingdefault}{\sfdefault}{m}{sl}
\SetMathAlphabet{\mathsfit}{bold}{\encodingdefault}{\sfdefault}{bx}{n}

\usepackage[T1]{fontenc}
\usepackage{amsmath,amssymb,amsthm,mathtools,bm}
\usepackage{booktabs}
\usepackage{array}
\usepackage{tabularx}
\usepackage{longtable}
\usepackage{graphicx}
\usepackage{stfloats}
\usepackage{cuted}
\newcommand{\dvlCompactFigure}[1]{\scalebox{1}[0.90]{#1}}
\usepackage{wrapfig}
\usepackage{needspace}   %
\usepackage{microtype}
\usepackage[table]{xcolor}
\usepackage{tikz}
\usepackage{xfp}
\usepackage{hyperref}
\usepackage{url}
\definecolor{linkblue}{RGB}{0,70,140}
\definecolor{modulejoint}{HTML}{D9806D}
\definecolor{moduleq}{HTML}{4D8B96}
\definecolor{modulek}{HTML}{A3C9CC}
\definecolor{modulev}{HTML}{8FB7A6}
\definecolor{moduleo}{HTML}{CEAE73}
\definecolor{moduleffn}{HTML}{C9958D}
\definecolor{lightblue}{RGB}{244,249,253}
\input{table_style}
\hypersetup{colorlinks=true,linkcolor=linkblue,citecolor=linkblue,urlcolor=linkblue}
\usepackage{letltxmacro}
\LetLtxMacro{\origtextbf}{\textbf}
\DeclareRobustCommand{\textbf}[1]{\ifvmode\noindent\fi\origtextbf{#1}}
\makeatletter
\renewcommand\subsubsection{\@startsection{subsubsection}{3}{\z@}%
  {-0.4ex plus -0.2ex minus -0.1ex}%
  {0.1ex plus 0.1ex}%
  {\normalfont\normalsize\bfseries}}
\makeatother

\newcommand{\jwl}{DV-Lens}
\newcommand{\jwlkm}{DV-Complexity}
\newcommand{\dvl}{DV-Lens}
\newcommand{\dvlkm}{DV-Complexity}

\newcommand{\WU}{W_{\!U}}

\newcommand{\ours}{Downstream Vocabulary Lens}
\newtheorem{theorem}{Theorem}[section]
\newtheorem{proposition}[theorem]{Proposition}
\newtheorem{lemma}[theorem]{Lemma}
\newtheorem{corollary}[theorem]{Corollary}

\theoremstyle{remark}

\newcommand{\norm}[1]{\left\lVert #1\right\rVert}

\makeatletter
\newcommand{\aliaslabel}[2]{%
  \expandafter\global\expandafter\let\csname r@#1\expandafter\endcsname\csname r@#2\endcsname}
\makeatother
\makeatletter
\newcommand{\extLabel}[2]{\@ifundefined{r@#1}{\global\@namedef{r@#1}{#2}}{}}

\makeatother
\makeatletter
\providecommand\xdocCounters[4]{}
\providecommand\extCounters[4]{}
\AtEndDocument{\@ifundefined{c@algorithm}{\def\xdoc@alg{0}}{\edef\xdoc@alg{\the\c@algorithm}}%
  \immediate\write\@auxout{\string\xdocCounters{\the\c@figure}{\the\c@table}{\the\c@equation}{\xdoc@alg}}}
\makeatother

\definecolor{suppaccent}{HTML}{4F8A75}
\input{supp_style}

\title{DV-Lens: Revealing the Functional Organization of Language Model Parameters}

\author{%
Chenhang Cui$^{1}$\quad Jian Yu$^{2}$\quad Jiani Luo$^{1}$\quad Shuyi Miao$^{3}$\quad Dingrui Liu$^{4}$\\
Rui Huang$^{5}$\quad Xiaohao Liu$^{1}$\quad Fei Shen$^{1}$\quad An Zhang$^{6}$\quad Tat-Seng Chua$^{1}$\\[0.5em]
{\small $^{1}$National University of Singapore\quad $^{2}$Nanjing University of Science and Technology}\\
{\small $^{3}$Beihang University\quad $^{4}$University of Electronic Science and Technology of China}\\
{\small $^{5}$The University of Hong Kong\quad $^{6}$University of Science and Technology of China}
}

\begin{document}
\maketitle
\etocdepthtag.toc{main}

\begin{abstract}
Understanding the functional roles of learnable parameters is important for elucidating the internal mechanisms of Large Language Models (LLMs).
Existing parameter-level methods mostly read weights through direct vocabulary projection or learned dictionaries, without accounting for how later layers transform a parameter's effect before it reaches the output.
To this end, we introduce the downstream vocabulary lens (DV-Lens), which propagates native parameter columns of attention (Q/K/V/O) and feed-forward modules through prompt-averaged downstream Jacobians and reads them out as signed responses over the output vocabulary.
We further introduce downstream vocabulary complexity (DV-Complexity), which groups parameters by these readouts and measures how much of the original weights the groups fail to reconstruct.
Finite-difference tests show that DV-Lens predicts the sign of local logit changes with 98.0\% agreement across cases from nine models, so the readouts indicate which tokens a parameter actually promotes or suppresses.
Guided by these readouts, parameter ablation, steering, and swapping shift target-token probabilities in the predicted directions across 21 models.
At the model level, DV-Complexity reaches a Spearman correlation of 0.904 with benchmark-based capability rankings across 48 LLMs, providing a structural signal of capability computed without benchmark evaluation.
\end{abstract}

\section{Introduction}
\label{sec:intro}

Large Language Models (LLMs) excel at language understanding, generation, and reasoning~\citep{suzgun2022bbh,wang2024mmlupro}, yet how their learned parameters support these capabilities remains unclear~\citep{belrose2023tuned,ghandeharioun2024patchscopes}. 
Understanding parameter functions requires examining how directions in different Transformer modules influence outputs through downstream computation~\citep{geva2021ffn,katz2024backward}. 
Beyond individual parameters, their functional organization may reveal structural properties associated with model capability~\citep{eilertsen2020classifying,dravid2024interpreting}. 
This motivates two connected questions: how can we characterize and verify the downstream effects of native parameter directions, and how is the resulting parameter organization related to model capability?

Existing interpretability research provides different views of these questions. Activation-based approaches analyze input-dependent representations using vocabulary readouts, contextual probing, and sparse feature decompositions~\citep{belrose2023tuned,ghandeharioun2024patchscopes,huben2024sae,dunefsky2024transcoders}, while Activation Patching tests their behavioral roles through interventions~\citep{zhang2023activationpatching}. 
Parameter-oriented studies interpret feed-forward weights as key-value memories and investigate the parameters involved in factual associations~\citep{geva2021ffn,meng2022rome}. For example, WeightLens~\citep{golimblevskaia2025weightlens} derives feature descriptions from learned transcoder weights, whereas ROTATE~\citep{avrahamy2026rotate} decomposes FFN weights into vocabulary-aligned directions. 
These advances demonstrate that vocabulary readouts can expose meaningful structure in model internals. 
However, they usually interpret parameters without considering how their effects pass through later computations to the model output. This motivates us to study how semantic information is organized across parameter dimensions and whether this organization is related to model capability.

\begin{figure*}[t]
  \centering
\dvlCompactFigure{\includegraphics[width=0.98\linewidth]{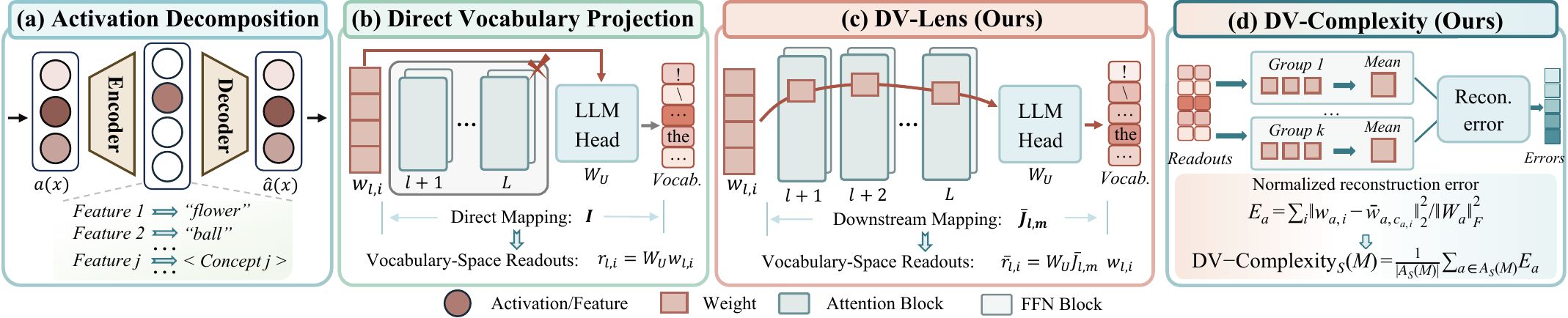}}
  \vspace{-0.1cm}
\caption{\textbf{From Parameter Readouts to Model-Level Complexity.}
(a) Activation decomposition explains input-dependent activations with learned features.
(b) Direct projection reads weights through the LLM head alone.
(c) \dvl{} propagates weights through the downstream mapping before vocabulary readout.
(d) \dvlkm{} groups normalized readouts (assignments $c_{a,i}$), reconstructs original weights with their group means, and averages energy-normalized matrix errors $E_a$. Groups, weights, and errors are schematic.}
  \label{fig:teaser}
\end{figure*}

To address these questions, we introduce the downstream vocabulary lens (\jwl{}), which uses native parameter columns as the unit of analysis. For attention Q/K/V/O projections and FFN modules, we estimate downstream Jacobian mappings over a reference prompt set and propagate parameter directions into the final vocabulary space. 
The resulting signed readouts characterize the average local logit response along each direction, placing parameters from different modules in a common output representation; Figure~\ref{fig:teaser}(a--c) contrasts this with activation decomposition and direct vocabulary projection.
Randomized controls and finite-difference tests assess these readouts, while parameter ablation, steering, and source-to-target swapping across 21 models provide behavioral evidence for predicted changes in target-token probabilities under controlled settings. 
To extend the analysis from individual parameter functions to parameter organization, we further introduce downstream vocabulary complexity (\jwlkm{}). As shown in Figure~\ref{fig:teaser}(d), this metric groups parameters by their normalized vocabulary readouts and measures the normalized reconstruction residual of the original weights within these groups, quantifying structural variation not captured by coarse functional roles. Across 48 language models, its joint-parameter score achieves a Spearman correlation of $0.904$ with benchmark-based capability rankings.
 The main contributions are summarized as follows:
\begin{itemize}
    \item We introduce \jwl{}, a parameter-level interpretability framework that uses module-specific downstream mappings to characterize native parameter directions across attention and FFN modules in a common vocabulary space.

    \item We provide complementary statistical and behavioral evidence for \jwl{} through matched random controls, finite-difference tests, and direct parameter interventions, linking its readouts to observed output changes under controlled settings.

    \item We introduce \jwlkm{} to quantify structural variation within parameter groups defined by downstream readouts, and show that the resulting model-level score is strongly associated with benchmark-based capability rankings across 48 language models.
\end{itemize}

\section{Related Work}

\textbf{Interpretability Methods for Language Models.}
Transformer interpretability studies representations, activations, and computational circuits~\citep{vaswani2017attention}. Tuned Lens and Patchscopes decode intermediate representations~\citep{belrose2023tuned,ghandeharioun2024patchscopes}. Jacobian Lens reads intermediate representations through prompt-averaged downstream Jacobians~\citep{gurnee2026jlens}. Attribution, causal mediation, and Activation Patching measure contributions to model behavior~\citep{sundararajan2017axiomatic,vig2020causal,zhang2023activationpatching}. Sparse Autoencoders, transcoders, and ICA-based methods expose features and directions~\citep{huben2024sae,dunefsky2024transcoders,icalens2026}, complementing studies of superposition, truth directions, and learned world representations~\citep{elhage2022superposition,marks2023geometry,li2023world}. In vision and vision-language settings, DNA identifies forgery-discriminative units and LAKE isolates sparse anomaly-sensitive neurons~\citep{dou2026dna,li2026latent}; shared-neuron analyses investigate common inference computations across language and vision-language models~\citep{cui2026llmsvlms}. Attention-head analyses and circuit discovery identify specialized computations~\citep{voita2019heads,michel2019sixteen,olsson2022incontext,wang2022ioi,conmy2023automated,nanda2023grokking}. These studies expose readable features and selected computational mechanisms, but their units of interpretation are typically activations, learned dictionaries, or task-specific circuits. In contrast, \dvl{} characterizes native parameter columns through module-specific downstream mappings and prompt-averaged vocabulary responses.

\textbf{Parameter Semantics, Localization, and Editing.}
Feed-forward layers act as key-value memories whose value directions promote vocabulary concepts~\citep{geva2021ffn,geva2022promote}; knowledge-neuron analyses connect individual neurons to factual predictions~\citep{dai2022knowledge}. Backward Lens projects gradients into vocabulary space to inspect information written into parameters~\citep{katz2024backward}. WeightLens derives descriptions from transcoder weights, whereas ROTATE extracts vocabulary-aligned directions from feed-forward weights without forward passes~\citep{golimblevskaia2025weightlens,avrahamy2026rotate}. Causal Tracing and model editing localize and modify factual associations~\citep{meng2022rome,meng2023memit}. In text-to-image models, OrthoEraser uses coupled-neuron orthogonal projection to erase sensitive concepts while preserving selected benign features~\citep{shi2026orthoeraser}. Such analyses connect weight-space structure to semantic content and targeted factual updates, while leaving open how native directions across different modules propagate through subsequent computation. Their readouts and edits therefore provide complementary evidence rather than a universal, context-independent interpretation of every parameter. \dvl{} instead tests predicted output changes through direct interventions on the identified parameter signals in native weight columns.

\section{The \dvl{} Framework}
\label{sec:method}
\ours{} (\dvl{}, Figure~\ref{fig:teaser}(c)) has three connected stages.
Section~\ref{sec:dvl_readout} derives the local logit response of one
parameter column on one prompt, showing why it must be traced through the
downstream computation rather than read through the LLM head alone.
Section~\ref{sec:dvl_computation} averages this mapping over a reference
prompt set, giving a signed vocabulary readout for every Q/K/V/O and FFN
column. The readouts are then used at the parameter level for ablation,
steering, and swapping (Section~\ref{sec:dvl_intervention}), testing whether
they predict output changes, and at the model level to define \dvlkm{}
(Figure~\ref{fig:teaser}(d); Sections~\ref{sec:dvl_kmeans} and~\ref{sec:jkm_theory}).

\subsection{Tracing Parameter Effects to the Output}
\label{sec:dvl_readout}

We begin with a local calculation that motivates \dvl{}.
Given an input $x$, a Transformer module $m$ at layer $\ell$, and a
learned parameter column $w_{\ell,m,i}$, we ask how increasing the
activation associated with this parameter changes the model's final
output logits over the vocabulary.

\textbf{Learned Parameters and Input-Dependent Activations.}
Consider a linear projection
$    h_{\ell,m}(x)
    =
    W_{\ell,m}f_{\ell,m}(x),
$
where
$W_{\ell,m}
=[w_{\ell,m,1},\ldots,w_{\ell,m,n_m}]$,
$f_{\ell,m}(x)$ is the input activation, and $n_m$ is its width.
The module output can be written as
\begin{equation}
    h_{\ell,m}(x)
    =
    \sum_{i=1}^{n_m}
    w_{\ell,m,i}f_{\ell,m,i}(x).
    \label{eq:parameter_decomposition}
\end{equation}

Thus, $w_{\ell,m,i}$ specifies the direction written by parameter
column $i$, while $f_{\ell,m,i}(x)$ determines how strongly that
direction participates in the current computation.  This columnwise
view applies uniformly to FFN and attention projection matrices, including Q, K, V, and O.

\textbf{Propagation through Subsequent Model Computation.}
Now increase the corresponding activation by a small amount, $f_{\ell,m,i}(x)
    \rightarrow
    f_{\ell,m,i}(x)+\epsilon.$
Eq.~\ref{eq:parameter_decomposition} gives
$
    h_{\ell,m}(x;\epsilon)
    =
    h_{\ell,m}(x)+\epsilon w_{\ell,m,i}.
$ Let
 $
    z(x)
    =
    F_{\ell,m\rightarrow L}(h_{\ell,m}(x);x)
$
denote the final normalized representation before the LLM
head, where $F_{\ell,m\rightarrow L}$ contains the remaining
computation after module $(\ell,m)$ and $L$ denotes the final layer.
We define the downstream Jacobian as
 \begin{equation}
    J_{\ell,m\rightarrow L}(x)
    =
    \frac{\partial z(x)}
         {\partial h_{\ell,m}(x)}.
 \end{equation}

By the chain rule, we have
$
    \frac{\partial z(x)}
         {\partial f_{\ell,m,i}}
    =
    J_{\ell,m\rightarrow L}(x)w_{\ell,m,i}.
    \label{eq:parameter_final_response}
$

With
$\operatorname{logits}(x)=\WU z(x)+b_U$, where $\WU$ is the
unembedding matrix and $b_U$ is the output bias, the logit response is
\begin{equation}
    \frac{\partial\,\operatorname{logits}(x)}
         {\partial f_{\ell,m,i}}
    =
    \WU
    J_{\ell,m\rightarrow L}(x)
    w_{\ell,m,i}.
    \label{eq:dvl_derivative}
\end{equation}

Equivalently, for a small increment $\epsilon$,
\begin{equation}
    \begin{aligned}
    &\operatorname{logits}(x;\epsilon)-\operatorname{logits}(x;0)\\
    &\qquad=\epsilon\WU J_{\ell,m\rightarrow L}(x)w_{\ell,m,i}
    +O(\epsilon^2).
    \end{aligned}
    \label{eq:dvl_taylor}
\end{equation}

We define the prompt-local unit-response readout as
$
    r_{\ell,m,i}(x)
    =
    \WU J_{\ell,m\rightarrow L}(x)w_{\ell,m,i}.
    \label{eq:dvl_prompt_readout}
$
Eq.~\ref{eq:dvl_derivative} characterizes the response per unit
change of $f_{\ell,m,i}(x)$.  The first-order contribution already
present in the prompt computation is instead
$
    c_{\ell,m,i}(x)
    =
    f_{\ell,m,i}(x)r_{\ell,m,i}(x).
    \label{eq:dvl_realized_contribution}
$
The activation coefficient sets the vocabulary readout's orientation and
scale; changing its sign reverses the coordinate ordering.
\subsection{Prompt-Averaged Vocabulary Readouts}
\label{sec:dvl_computation}

The local response in Eq.~\ref{eq:dvl_derivative} depends on the
computation induced by a particular input $x$. \dvl{} turns this
input-dependent response into a common parameter representation by
estimating the downstream mapping over a fixed prompt set
$\mathcal D$ shared across all modules.

\textbf{Estimating the Downstream Mapping.}
For each module $(\ell,m)$, the mean downstream Jacobian is
\begin{equation}
    \bar J_{\ell,m}
    =
    \frac{1}{|\mathcal D|}
    \sum_{x\in\mathcal D}
    J_{\ell,m\rightarrow L}(x).
    \label{eq:average_jacobian}
\end{equation}

For sequence-valued module outputs, the Jacobian is first aggregated
over the fixed set of valid prediction positions and then averaged
equally across prompts. This gives a common downstream mapping from the
output of module $(\ell,m)$ to the final model representation.

\textbf{\dvl{} Readout.}
The \emph{\dvl{} readout} of column $w_{\ell,m,i}$ is
\begin{equation}
    \bar r_{\ell,m,i}
    =
    \WU
    \bar J_{\ell,m}
    w_{\ell,m,i}
    \in
    \mathbb{R}^{|\mathcal V|}.
    \label{eq:dvl}
\end{equation}

Here $\mathcal V$ is the output vocabulary. The entries of
$\bar r_{\ell,m,i}$ describe the prompt-set average local
response to a unit change of the corresponding activation coefficient.
Their magnitudes identify vocabulary coordinates most affected along this
parameter direction.  The averaged readout is a parameter footprint, not
the activation-weighted contribution realized on any individual prompt.
Vocabulary-space readouts have been widely used to interpret
intermediate representations and model parameters
~\citep{belrose2023tuned,geva2021ffn,katz2024backward,
golimblevskaia2025weightlens,avrahamy2026rotate}.
Unlike a direct projection of an intermediate representation or weight,
\dvl{} first propagates the parameter through the subsequent model
computation via $\bar J_{\ell,m}$ and only then reads its effect in the
final output vocabulary. Direct vocabulary projection, $\WU w_{\ell,m,i}$,
corresponds to replacing $\bar J_{\ell,m}$ with the identity and therefore
ignores every layer after module $(\ell,m)$.

\textbf{Readouts across a Parameter Matrix.}
For an entire weight matrix
$W_{\ell,m}=[w_{\ell,m,1},\ldots,w_{\ell,m,n_m}]$,
the \dvl{} readouts of all parameter columns can be written jointly as
\begin{equation}
    R_{\ell,m}
    =
    \WU
    \bar J_{\ell,m}
    W_{\ell,m}
    \in
    \mathbb{R}^{|\mathcal V|\times n_m}.
    \label{eq:dvl_matrix}
\end{equation}

Its $i$-th column is exactly the \dvl{} readout of
$w_{\ell,m,i}$,
$R_{\ell,m}[:,i]=\bar r_{\ell,m,i}$.
We estimate $\bar J_{\ell,m}$ once per module using automatic
differentiation and read out all columns jointly as
$\WU\bar J_{\ell,m}W_{\ell,m}$. We apply the same procedure
independently to each eligible Q/K/V/O and FFN weight matrix, placing
their learned parameters in a common representation space defined by
the final output vocabulary.

\subsection{Parameter Signals for Targeted Token Control}
\label{sec:dvl_intervention}

The vocabulary readout maps learned parameters to output tokens. For a
target concept, we identify its associated parameter signals and use them
to control the corresponding tokens through direct weight edits.

\textbf{Representing a Concept in the Output Vocabulary.}
Let $\mathcal T_c$ denote the token set associated with concept $c$.
For intervention analysis, we incorporate the learned
final-normalization scale into the unembedding as
$\widetilde W_U=\WU\Gamma$, where $\Gamma$ is its diagonal scale matrix.
We represent concept $c$ by
$u_c=|\mathcal T_c|^{-1}
\sum_{v\in\mathcal T_c}\widetilde W_U[v,:]^\top$.
For $C$ concepts, we distinguish each from the others using the
one-versus-rest readout
\begin{equation}
    s_c
    =
    u_c
    -
    \frac{1}{C-1}
    \sum_{c'\neq c}
    u_{c'}.
    \label{eq:semantic_contrast}
\end{equation}

These readouts satisfy $\sum_{c=1}^{C}s_c=0$ and therefore span a subspace of at most
$C-1$ dimensions.

\textbf{Mapping Output Concepts to Learned Parameters.}
Using a fixed prompt set $\mathcal D_J$ for Jacobian estimation, let
$\bar J_{\ell,m}
=
|\mathcal D_J|^{-1}
\sum_{x\in\mathcal D_J}
J_{\ell,m\rightarrow L}(x)$.
We map concept $c$ to module $(\ell,m)$ through
\begin{equation}
    d_{\ell,m,c}
    =
    \bar J_{\ell,m}^{\top}s_c
    \in\mathbb{R}^{d_m},
    \label{eq:concept_parameter_signal}
\end{equation}
which we call the \emph{parameter signal} for concept $c$; $d_m$ is the
module output width.
For a parameter column $w_{\ell,m,i}$, its  association with this
concept is
$
    \beta_{\ell,m,i,c}
    =
    d_{\ell,m,c}^{\top}w_{\ell,m,i}
    =
    s_c^{\top}\bar J_{\ell,m}w_{\ell,m,i}.
    \label{eq:semantic_response}
$
The last expression has the same form as the \dvl{} readout: it traces a
learned parameter through subsequent model computation and evaluates
its effect with respect to the target concept readout $s_c$.

\textbf{Ablation and Steering.} Let $\widehat d=d/\|d\|_2$ denote the unit parameter signal and
$\lambda$ the intervention strength (denoted $\alpha$ in the experiments).
For a source concept $c$, we reduce its associated signal in each
selected parameter as
$
    w_{\ell,m,i}^{\,\prime}
    =
    w_{\ell,m,i}
    -
    \lambda
    \widehat d_{\ell,m,c}
    \left(
        \widehat d_{\ell,m,c}^{\top}w_{\ell,m,i}
    \right).
    \label{eq:ablation}
$
For a target concept, we increase its associated signal as
$
    w_{\ell,m,i}^{\,\prime}
    =
    w_{\ell,m,i}
    +
    \lambda
    \|w_{\ell,m,i}\|_2
    \widehat d_{\ell,m,c}.
    \label{eq:steering}
$

\textbf{Source-to-Target Swapping.}
Swapping combines source removal and target addition.  We select the
source-removal set $\mathcal S_{\mathrm{src}}$ and target-addition set
$\mathcal S_{\mathrm{tgt}}$ independently and apply
\begin{equation}
\begin{gathered}
w_{\ell,m,k}^{\,\prime}=\\[-2pt]
\begin{cases}
w_{\ell,m,k}-\lambda_{\mathrm{src}}\widehat d_{\ell,m,\mathrm{src}}
\bigl(\widehat d_{\ell,m,\mathrm{src}}^\top w_{\ell,m,k}\bigr),
& k\in\mathcal S_{\mathrm{src}}, \\[4pt]
w_{\ell,m,k}+\lambda_{\mathrm{tgt}}\|w_{\ell,m,k}\|_2
\widehat d_{\ell,m,\mathrm{tgt}},
& k\in\mathcal S_{\mathrm{tgt}}.
\end{cases}
\end{gathered}
\label{eq:swap}
\end{equation}

\subsection{\dvlkm{} and Model Capability}
\label{sec:dvl_kmeans}
Beyond individual parameters, we test whether \dvl{} readouts organize
learned weights in a way that relates to model capability. \dvlkm{} groups
parameters by their readouts and measures the corresponding structure in
the original weights.
Figure~\ref{fig:teaser}(d) summarizes this construction: normalized readouts
determine group membership, each original weight column is reconstructed
by its group's mean, and the residual is normalized by total matrix energy.
Averaging these matrix-level errors produces one score for each model and
module scope.

\textbf{DV-Lens-Guided Parameter Grouping.}
For a weight matrix $a=(\ell,m)$, write
$W_a=[w_{a,1},\ldots,w_{a,n_a}]$ and let
$R_a=\WU\bar J_aW_a$ denote its \dvl{} readout matrix.  Its $i$-th column
$r_{a,i}=R_a[:,i]$ is the  vocabulary readout associated with
parameter $w_{a,i}$.  We normalize each readout as
$\widetilde r_{a,i}
=r_{a,i}/(\|r_{a,i}\|_2+\varepsilon)$,
where $\varepsilon>0$ stabilizes normalization. With $k$ clusters and
readout centroids $\mu_{a,c}$, the cluster assignments of the $n_a$ columns are
\begin{equation}
    c_{a,i}^{\mathrm{DVL}}
    =
    \arg\min_{c\in\{1,\ldots,k\}}
    \left\|
        \widetilde r_{a,i}-\mu_{a,c}
    \right\|_2^2,
    \qquad k=32.
    \label{eq:dvl_assignment}
\end{equation}

\textbf{Parameter Organization in Learned Weights.}
The cluster assignments are determined by \dvl{} readouts, while their
organization is evaluated in the original learned parameter space.
For each cluster $c$ with members
$\mathcal I_{a,c}=\{i:c_{a,i}^{\mathrm{DVL}}=c\}$, let
$\bar w_{a,c}=|\mathcal I_{a,c}|^{-1}\sum_{i\in\mathcal I_{a,c}}w_{a,i}$.
We define the \mbox{matrix-level reconstruction error as}
\begin{equation}
    E_a^{\mathrm{DVL}\rightarrow\mathrm{raw}}
    =
    \frac{1}{\sum_{j=1}^{n_a}\|w_{a,j}\|_2^2}
    \sum_{i=1}^{n_a}
    \left\|
        w_{a,i}
        -
        \bar w_{a,c_{a,i}^{\mathrm{DVL}}}
    \right\|_2^2.
    \label{eq:matrix_dvl_kmeans}
\end{equation}

Smaller values indicate tighter grouping of learned weights with similar
\dvl{} readouts.

\textbf{Model-Level Parameter Organization.}
For model $M$ and module scope $\mathcal S$, let
$\mathcal A_{\mathcal S}(M)$ denote its eligible weight matrices.  We
define
\begin{equation}
    \operatorname{DV\mbox{-}Complexity}_{\mathcal S}(M)
    =
    \frac{1}{|\mathcal A_{\mathcal S}(M)|}
    \sum_{a\in\mathcal A_{\mathcal S}(M)}
    E_a^{\mathrm{DVL}\rightarrow\mathrm{raw}}.
    \label{eq:model_dvl_kmeans}
\end{equation}

This aggregates the matrix-level organization into a single structural
score for each model.

\textbf{Full-Cohort Relation to Model Capability.}
Let $\mathcal C=\{M_j\}_{j=1}^{48}$ be the fixed checkpoint cohort and
let $y_j$ denote the capability score of $M_j$.  For every module scope
$\mathcal S\in\{\mathrm{Joint},Q,K,V,O,\mathrm{FFN}\}$, define the
descending capability and structural ranks
\begin{equation}
    \begin{aligned}
    R_y(j)&=\operatorname{rank}_{\downarrow}(y_j),
    \\
    R_{\mathcal S}(j)&=\operatorname{rank}_{\downarrow}
    \!\left(\operatorname{DV\mbox{-}Complexity}_{\mathcal S}(M_j)\right),
    \end{aligned}
    \label{eq:dvl_capability_ranks}
\end{equation}
where rank 1 denotes the largest value and ties receive their standard
average rank.  We report Spearman association as the Pearson correlation
of these two complete rank vectors,
\begin{equation}
    \rho_{\mathcal S}
    =
    \operatorname{Corr}_{\mathrm{Pearson}}
    \left(
        \left\{R_{\mathcal S}(j)\right\}_{j=1}^{48},
        \left\{R_y(j)\right\}_{j=1}^{48}
    \right).
    \label{eq:dvl_capability_correlation}
\end{equation}

The capability variable, cohort membership, and rank convention are shared
across all six scopes; no scope-specific checkpoint filtering or cohort
selection is performed.  This estimand tests cross-sectional association
between the organization exposed by \dvl{} and the fixed capability measure across the 48 checkpoints.
 
\subsection{Theoretical Interpretation of \dvlkm{}}
\label{sec:jkm_theory}
We interpret \dvlkm{} as measuring functional parameter structure that
cannot be compressed into a small set of roles without affecting the
predictive objective. Specifically, the score quantifies the normalized
parameter energy removed by collapsing weights within each \dvl{} role,
which can be related to the corresponding increase in predictive loss.
This establishes \dvlkm{} as a measure of loss-supported functional
parameter richness. Appendix~\ref{app:score_theory} derives this
interpretation and provides the proofs.

\section{Experiments}
\label{sec:experiments}
We test \dvl{} readouts with matched random controls and finite differences, then assess their behavioral effects through parameter ablation, steering, and swapping. Model-level analyses examine capability association and \dvlkm{} stability under complete-benchmark exposure.
\begin{figure*}[t]
  \centering
  \dvlCompactFigure{\includegraphics[width=\textwidth]{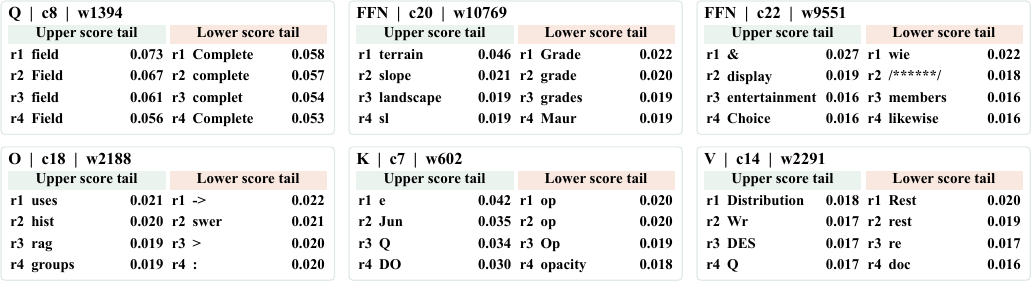}}
  \caption{\textbf{Same-Model Centroids Expose Vocabulary-Tail Groups Whose Readability Varies.}
  Six fixed Mistral-7B clusters. Each header reads module\,|\,cluster\,|\,parameter
  column: \textsf{c}$k$ is the $k$-th \dvlkm{} cluster (Eq.~\ref{eq:dvl_assignment})
  and \textsf{w}$i$ the parameter column $w_{a,i}$ inspected in it.
  Upper/lower tails list the tokens with the largest/smallest readout scores;
  \textsf{r}$j$ is the rank within the tail (\textsf{r1} = most extreme), and
  numbers are score magnitudes. Selected cases illustrate organization, not
  population readability.}
  \label{fig:token_centroids}
\end{figure*}
\subsection{Statistical and Functional Validation}
\label{sec:validity}

We test non-random vocabulary structure and local response prediction.
The original validation uses five models: Llama-3.2-3B~\citep{grattafiori2024llama3}, Mistral-7B~\citep{jiang2023mistral},
OLMo-2-7B~\citep{olmo2025olmo2}, Granite-3.3-8B~\citep{granite2024granite3}, and Qwen2.5-0.5B~\citep{qwen2025qwen25}.
We further extend the statistical comparison to 21 models and the
finite-difference test to nine models. CI denotes confidence interval.
Complete matched-null results, including module-level and family-level
comparisons, are provided in Appendix~\ref{app:matched_null_breadth}.

\subsubsection{Non-Random Vocabulary Structure}
\label{sec:matched_null_statistical_exceptionality}
\definecolor{dvlgrlo}{HTML}{7CC294}\definecolor{dvlgrhi}{HTML}{145236}
\providecommand{\dvlShade}[1]{\edef\dvlCIp{\fpeval{round(max(0,min(100,100*(#1-0.630)/(5.391-0.630))))}}%
  \colorlet{dvlcic}{dvlgrhi!\dvlCIp!dvlgrlo}}
\providecommand{\dvlDK}[1]{\dvlShade{#1}\textcolor{dvlcic}{+#1}}
\providecommand{\dvlCIw}{78pt}
\providecommand{\dvlCIx}[1]{\fpeval{#1/10.5}*\dvlCIw}
\providecommand{\dvlCIzero}{\draw[overlay,dvlmute!70,line width=.35pt,dash pattern=on 1.2pt off 1.2pt] (0,-6.3pt) -- (0,6.3pt);}
\providecommand{\dvlCIbar}[5]{\dvlShade{#2}\begin{tikzpicture}[baseline=-0.6ex]
  \useasboundingbox (0,-3pt) rectangle (\dvlCIw,3pt);
  \fill[black!6] (0,-#4) rectangle (\dvlCIw,#4);
  \fill[dvlcic] (0,-#4) rectangle ({\dvlCIx{#2}},#4);
  \dvlCIzero
  \draw[black!72,line width=#5] ({\dvlCIx{#1}},0) -- ({\dvlCIx{#3}},0)
    ({\dvlCIx{#1}},-1.7pt) -- ({\dvlCIx{#1}},1.7pt) ({\dvlCIx{#3}},-1.7pt) -- ({\dvlCIx{#3}},1.7pt);
\end{tikzpicture}}
\providecommand{\dvlCI}[3]{\dvlCIbar{#1}{#2}{#3}{2.7pt}{.5pt}}
\providecommand{\dvlCIMean}[3]{\dvlCIbar{#1}{#2}{#3}{3.1pt}{.65pt}}
\providecommand{\dvlCIAxis}{\begin{tikzpicture}[baseline=-0.6ex]
  \useasboundingbox (0,-3pt) rectangle (\dvlCIw,3pt);
  \draw[overlay,dvlmute!70,line width=.35pt,dash pattern=on 1.2pt off 1.2pt] (0,-6.3pt) -- (0,1.2pt);
  \fill[black!45] (0,0) circle (1.5pt);
  \foreach \v in {5,10} \node[inner sep=0pt,text=dvlmute,font=\fontsize{6.5pt}{7pt}\selectfont] at ({\dvlCIx{\v}},0) {\v};
  \node[inner sep=0pt,anchor=west,text=dvlmute,font=\fontsize{6.5pt}{7pt}\selectfont] at (3pt,0) {0};
\end{tikzpicture}}
\begin{table}[t]
\centering
\dvlTableSetup
\fontsize{8pt}{9.5pt}\selectfont
\setlength{\tabcolsep}{2.0pt}
\renewcommand{\arraystretch}{1.32}
\caption{\textbf{Model-Level Matched-Null Comparison.}
$\Delta K=K_{\mathrm{DVL}}-K_{\mathrm{null}}$ measures the
excess-kurtosis advantage of \dvl{} over matched Gaussian directions.
Positive values favor \dvl{}. Bars: $\Delta K$ on a shared axis, darker for
larger values; whiskers: 95\% CIs; the Appendix lists all 21 models.}
\label{tab:null_validation}
\begin{tabularx}{\linewidth}{>{\hsize=1.2\hsize\raggedright\arraybackslash}X>{\hsize=.6\hsize\raggedleft\arraybackslash}X>{\hsize=1.2\hsize\raggedleft\arraybackslash}X}
\toprule
\dvlHead{Model} & \dvlHead{$\Delta K\,\uparrow$} & \dvlHead{95\% CI} \\
\midrule
\textcolor{black!65}{Random baseline}
& \textcolor{black!65}{0.000}
& \dvlCIAxis \\
\addlinespace[1pt]
\dvlModel{ai2}{OLMo-2-7B}        & \dvlDK{0.630} & \dvlCI{0.284}{0.630}{1.054} \\
\dvlModel{yi}{Yi-1.5-9B}        & \dvlDK{1.128} & \dvlCI{0.561}{1.128}{1.735} \\
\dvlModel{qwen}{Qwen2.5-0.5B}     & \dvlDK{1.488} & \dvlCI{0.787}{1.488}{2.316} \\
\dvlModel{qwen}{Qwen2.5-14B}      & \dvlDK{2.534} & \dvlCI{1.128}{2.534}{4.283} \\
\dvlModel{qwen}{Qwen1.5-14B}      & \dvlDK{2.858} & \dvlCI{0.594}{2.858}{6.522} \\
\dvlModel{meta}{Llama-3.2-3B}     & \dvlDK{3.107} & \dvlCI{1.372}{3.107}{5.190} \\
\dvlModel{mistral}{Mistral-7B}       & \dvlDK{3.263} & \dvlCI{1.224}{3.263}{5.771} \\
\dvlModel{ibm}{Granite-3.3-8B}   & \dvlDK{5.391} & \dvlCI{1.714}{5.391}{10.192} \\
\midrule
\rowcolor{dvltint}
\textbf{Mean} & \dvlShade{2.550}\textcolor{dvlcic}{\textbf{+2.550}} & \dvlCIMean{1.626}{2.550}{3.574} \\
\bottomrule
\end{tabularx}
\end{table}

The Gaussian controls match \dvl{} directions in dimension, covariance
treatment, and norm. The extended analysis covers 21 models from
13 families, three depths, and Q/K/V/O/FFN. Each model-depth-module cell compares 12 concept-mean \dvl{} directions
with 192 matched Gaussian directions. Table~\ref{tab:null_validation} shows the eight models
that also appear in the finite-difference study; complete results are reported
in Appendix~\ref{app:matched_null_breadth}.
Table~\ref{tab:null_validation} shows that all eight displayed model means are
above the matched-random baseline, although the size of the difference varies
across models. Across these eight models, the mean excess-kurtosis difference is $+2.550$
(model-bootstrap 95\% CI $[1.626,3.574]$). The original
prompt-level validation shows the same direction: the excess-kurtosis
difference is $+2.128$ with 95\% CI $[0.609,3.671]$, and the Top-1\%
$\ell_2$-energy difference is $+0.0191$ with 95\% CI
$[0.0119,0.0263]$.
Together, these results show that \dvl{} directions have stronger heavy-tail
and concentration structure than matched random directions.
The 21-model
concept-level analysis and the original prompt-level analysis use different
aggregation units and are therefore reported separately rather than pooled.

\textbf{What the Tails Contain.}
In six fixed Mistral-7B clusters (Figure~\ref{fig:token_centroids};
Appendix~\ref{app:frozen_functional_inspection}), some tails form readable
groups (field vs.\ complete, terrain vs.\ grade) and others are mostly
fragments: \dvl{} readouts support local interpretability without assuming
that every parameter or cluster has one clean meaning in the output vocabulary.

\subsubsection{Finite-Difference Prediction}
\label{sec:finite_difference_prediction}
\suppressfloats[t]%
\begin{figure}[t]
  \centering
  \dvlCompactFigure{\includegraphics[width=\linewidth]{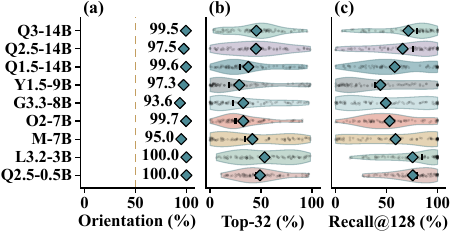}}
  \caption{\textbf{Exact \dvl{} Readouts Recover Finite-Difference Coordinate Orientation and High-Impact Vocabulary Regions over 720 Frozen Cases from Nine Models, with 80 Cases per Model.}
  \textbf{(a)} Model-mean coordinate-orientation agreement; the gold dashed line marks the 50\% random-orientation baseline.
  \textbf{(b)} Exact two-tail Top-32 overlap.
  \textbf{(c)} Two-tail Top-32 Recall@128.
  Model labels abbreviate family, version, and size: Q, Y, G, O, M, and L denote
  Qwen, Yi, Granite, OLMo, Mistral, and Llama (e.g., Q2.5-14B is Qwen2.5-14B).}
  \label{fig:fd_validation_top32}
\end{figure}

\begin{figure}[t]
  \centering
  \dvlCompactFigure{\includegraphics[width=0.95\linewidth]{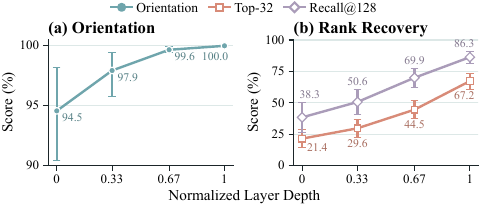}}
  \caption{\textbf{Finite-Difference Prediction by Layer Depth.}
  The same 720 cases are grouped by the four sampled layer positions, with 180
  cases per depth. \textbf{(a)} Orientation agreement, separated to show its
  near-saturated depth trend. \textbf{(b)} The rank-recovery metrics, exact
  two-tail Top-32 overlap and two-tail Recall@128, on the same depths. Curves
  average nine models; error bars show 95\% model-bootstrap CIs. Horizontal
  positions use mean  depths.}
  \label{fig:fd_depth}
\end{figure}
\begin{figure*}[!t]
  \centering
  {\fontsize{8}{9}\selectfont\input{paper_figures/Figure3_main_right_labels}}%
  \vspace{-0.1cm}
  \caption{\textbf{Conditional Decoded Success Rates for Twelve of the 21 Intervention Models across Six Module Scopes and Three Operations.}
  The eight columns in each bank are the edited-parameter fractions from $0.05\%$ to $5\%$; the bottom row is the matched random baseline. Model labels follow Figure~\ref{fig:fd_validation_top32}; DS-R1, F, and SL denote DeepSeek-R1-Distill-Qwen, Falcon, and StableLM.}
  \label{fig:intervention_budget}
\end{figure*}

We next test whether Eq.~\ref{eq:dvl_taylor} predicts the actual response
to small, controlled increments along parameter directions. The evaluation contains 720 fixed cases from
nine models, with 80 cases per model. For a score vector $s$, let $U_s^K$ and
$L_s^K$ denote the token IDs with the $K$ largest and smallest scores. We
measure coordinate-orientation agreement as the fraction of matching
predicted and observed coordinate signs. With FD denoting finite differences, define
$
R_{32\rightarrow K}
=
\bigl(
|U_{\mathrm{DVL}}^{32}\cap U_{\mathrm{FD}}^K|
+
|L_{\mathrm{DVL}}^{32}\cap L_{\mathrm{FD}}^K|
\bigr)/64
$
for $K\in\{32,128\}$.
For $K=32$, this is the exact two-tail Top-32 overlap. For $K=128$, it becomes
Top-32 Recall@128, which allows small rank changes while keeping the positive
and negative tails separate. Because enlarging the reference tail raises recall
by construction, we also report the matching F1 score,
$F1_{32\rightarrow K}=2n/(64+2K)$ with $n=64R_{32\rightarrow K}$ matched tokens.
For $K=32$ both sets contain 64 tokens, so precision, recall, and F1 all equal
the Top-32 overlap; for $K=128$, $F1_{32\rightarrow128}=0.4\,R_{32\rightarrow128}$.
The random expectations are $1/2$ for orientation
and $32/|\mathcal V|$ and $128/|\mathcal V|$ for Top-32 overlap and Recall@128, respectively.

Figure~\ref{fig:fd_validation_top32} shows that coordinate orientation is close
to one for all nine models. Exact Top-32 overlap is lower and varies more
across models, while every model improves under Recall@128. Since this gain
partly follows from the larger reference tail, we compare it with its chance
level, at most $128/|\mathcal V|<0.004$, rather than read it as higher
accuracy. Against this baseline, it indicates
that many exact-rank misses remain within the same high-impact vocabulary
region rather than moving to unrelated coordinates. Thus, \dvl{} predicts the
direction and broader coordinate structure of the local response more reliably
than the exact ordering within each tail: close orientation agreement can
coexist with changes in the ranking of high-impact tokens.

\begin{table}[t]
\centering
\dvlTableSetup
\fontsize{8pt}{9.5pt}\selectfont
\setlength{\tabcolsep}{3pt}
\renewcommand{\arraystretch}{1.16}
\caption{\textbf{Finite-Difference Validation.}
Orientation agreement is the fraction of coordinates whose predicted and
observed signs match. Top-32 overlap is the fraction of the 64 two-tail \dvl{}
Top-32 tokens (32 largest and 32 smallest scores) found in the matching
finite-difference Top-32 tail. Top-32 Recall@128 enlarges the
finite-difference tails to 128 tokens; Top-32 F1@128 is its F1 score.}
\label{tab:fd}
\begin{tabularx}{\linewidth}{>{\raggedright\arraybackslash}X*{4}{>{\raggedleft\arraybackslash}X}}
\toprule
\dvlHead{Module} & \dvlHead{\shortstack[r]{Orientation\\agreement}} & \dvlHead{\shortstack[r]{Top-32\\overlap}} & \dvlHead{\shortstack[r]{Top-32\\Recall@128}} & \dvlHead{\shortstack[r]{Top-32\\F1@128}} \\
\midrule
\dvlScope{moduleq}{Q} & 0.976 & 0.313 & 0.525 & 0.210 \\
\dvlScope{modulek}{K} & 0.947 & 0.209 & 0.375 & 0.150 \\
\dvlScope{modulev}{V} & 0.988 & 0.412 & 0.650 & 0.260 \\
\dvlScope{moduleo}{O} & 0.992 & 0.537 & 0.743 & 0.297 \\
\dvlScope{moduleffn}{FFN} & 0.997 & 0.565 & 0.770 & 0.308 \\
\midrule
\rowcolor{dvltint}
\textbf{Pooled} & \textbf{0.980} & \textbf{0.407} & \textbf{0.613} & \textbf{0.245} \\
\bottomrule
\end{tabularx}
\end{table}
\begin{figure*}[!t]
  \centering
  \renewcommand{\thesubfigure}{\alph{subfigure}}
  \captionsetup{aboveskip=2pt,belowskip=0pt}
  \captionsetup[subfigure]{font=footnotesize,position=top,justification=raggedright,singlelinecheck=false,skip=1pt}
  \begin{subfigure}[t]{0.493\textwidth}
    \subcaption{\dvlModel{meta}{Llama-3.1-8B}\hfill V: swap, 0.5\%}
    \dvlCompactFigure{\includegraphics[width=\linewidth]{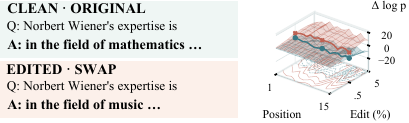}}
  \end{subfigure}
  \hfill
  \begin{subfigure}[t]{0.493\textwidth}
    \subcaption{\dvlModel{qwen}{Qwen2.5-3B}\hfill V: steer, 0.2\%}
    \dvlCompactFigure{\includegraphics[width=\linewidth]{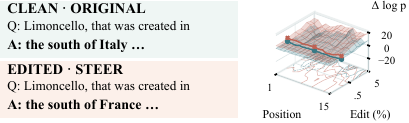}}
  \end{subfigure}
  \par\vspace{1pt}
  \begin{subfigure}[t]{0.493\textwidth}
    \subcaption{\dvlModel{meta}{Llama-3.2-3B}\hfill O: swap, 0.05\%}
    \dvlCompactFigure{\includegraphics[width=\linewidth]{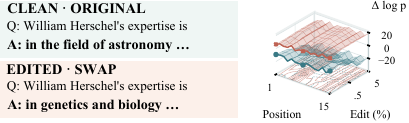}}
  \end{subfigure}
  \hfill
  \begin{subfigure}[t]{0.493\textwidth}
    \subcaption{\dvlModel{mistral}{Mistral-7B}\hfill V: steer, 0.05\%}
    \dvlCompactFigure{\includegraphics[width=\linewidth]{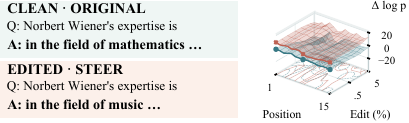}}
  \end{subfigure}
  \caption{\textbf{Decoded Intervention Cases.} Blue/orange cards show clean/edited answers; headers give model, module, operation, and fraction. Blue/orange surfaces show source/target log-probability changes (edited minus clean); positive/negative height indicates promotion/suppression.  Edit (\%) spans eight fractions logarithmically. Bold traces mark the header's fraction; floor contours project the surfaces.}
  \label{fig:intervention_cases_main}
\end{figure*}

Table~\ref{tab:fd} gives the corresponding module-level results. Pooled
orientation agreement is $0.980$, exact Top-32 overlap is $0.407$, and
Recall@128 is $0.613$, with a corresponding F1@128 of $0.245$. Orientation remains high across all modules, while
coordinate recovery is lowest for K and highest for O and FFN. The gain
from Top-32 overlap to Recall@128 further shows that many misses come from
local rank changes within the same high-impact tail.

Overall, these experiments provide two complementary checks: \dvl{} directions
show non-random vocabulary structure, often readable at the token level, and predict the local response to small
parameter changes. Additional structural and token-level analyses are provided
in Appendix~\ref{app:statistical_boundaries}. We next test whether these local
predictions translate into controlled changes in model behavior in
Section~\ref{sec:intervention}.

\textbf{Prediction across Layer Depths.}
Figure~\ref{fig:fd_depth} complements the model and module summaries by
grouping the four sampled layers according to their relative depth.
Each depth contains the same nine models, five modules, and 180 cases.
From the earliest to the latest sampled layer, orientation agreement
increases from $94.5\%$ to $99.96\%$, exact Top-32 overlap from $21.4\%$
to $67.2\%$, and Recall@128 from $38.3\%$ to $86.3\%$.
All nine models improve between these endpoints for each metric.
The early layers therefore already recover most coordinate signs,
while their high-impact token rankings are less precise.
Later layers improve exact and wider-tail recovery, locating the larger
ranking discrepancies in earlier sampled layers across all nine models.

\subsection{Behavioral Intervention}
\label{sec:intervention}
\newsavebox{\dvlExposureFigure}
\sbox{\dvlExposureFigure}{%
\begin{minipage}{\textwidth}
  \centering
  \includegraphics[width=0.98\textwidth]{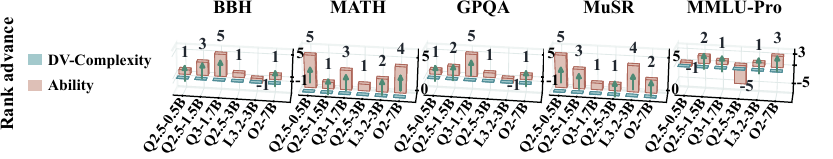}
  \captionof{figure}{\textbf{Benchmark-Derived Ability Ranks Advance While the Primary DV-Complexity Rank Remains Stationary under Complete-Benchmark Exposure.}
  Green arrows mark positive non-ceiling joint responses. Model labels follow Figure~\ref{fig:fd_validation_top32}.}
  \label{fig:fullbench_rank_stability}
\end{minipage}%
}
\begin{table*}[t]
\centering
\usebox{\dvlExposureFigure}
\end{table*}
\begin{figure}[t]
  \centering
  \includegraphics[width=\linewidth]{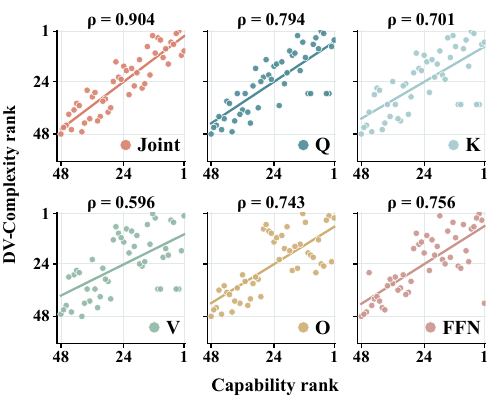}
  \captionsetup{aboveskip=1pt,belowskip=0pt}
  \caption{\textbf{Uniform Five-Task Capability Rank and DV-Complexity Rank on One Fixed 48-Checkpoint Cohort.}
  Panels are ordered Joint, Q, K, V, O, FFN; all six use the same 48 paired rank positions.}
  \label{fig:capability_shared48}
\end{figure}

We evaluate \textbf{ablate}, \textbf{steer}, and \textbf{swap} on 21 models and ten factual concepts across Q/K/V/O/FFN and their joint pool, using eight edited-parameter fractions ($0.05\%$--$5\%$).

Figure~\ref{fig:intervention_budget} reports the primary conditional decoded success rate for twelve models from DeepSeek-R1-Distill-Qwen, Falcon3, Llama-3.1/3.2, Mistral, Qwen1.5/2.5/3, and StableLM-2 (Appendix~\ref{app:decoded_all_models} shows all 21): whether the edited continuation satisfies the operation-specific answer-containment rule, among concept cells whose clean continuation contains the original gold answer. It is defined for 2,880 aggregate rows; 144 lack an eligible clean continuation. Appendix~\ref{app:probability_direction} reports future-token probability changes for all ten concepts.

Pooling concept counts at Joint scope, peak conditional decoded success rates are $61/66$ for ablation ($5\%$), $28/66$ for steering ($0.1\%$--$0.2\%$), and $23/66$ for swapping ($5\%$).

\rule{0pt}{\dimexpr\topskip+3pt\relax}Figure~\ref{fig:intervention_cases_main} connects decoded answer changes with source suppression and target promotion at the highlighted fraction. Each panel pairs the clean and edited answers with source- and target-token log-probability changes. In both examples, the answer change comes only from direct edits to the selected weight columns of the edited module. Detailed cases for six models appear in Appendix~\ref{app:qualitative_cases}, covering Q, K, V, O, FFN, and Joint with paired clean and edited answer excerpts; Appendix~\ref{app:random-control} also reports the matched random-column control with the same number of edited columns.

\subsection{Model-Level Capability and Benchmark-Exposure Diagnostics}
\label{sec:diagnostic}

\textbf{Cross-Sectional Capability Ranking.}
For 48 post-trained LLMs, \dvlkm{} uses 2,000 Jacobian-estimation prompts,
a rank-128 approximation, $k=32$, and all eligible layers and six module scopes. Capability is the unweighted mean over five Open LLM
Leaderboard v2 benchmarks~\citep{huggingfaceollbv2}:
BIG-Bench Hard (BBH)~\citep{suzgun2022bbh}, MATH~\citep{hendrycks2021math},
Graduate-Level Google-Proof Q\&A (GPQA)~\citep{rein2023gpqa}, Multistep Soft Reasoning (MuSR)~\citep{sprague2024musr},
and MMLU-Pro~\citep{wang2024mmlupro}, the enhanced Massive Multitask Language Understanding benchmark. Appendix~\ref{app:capability-score} specifies scoring, benchmark settings,
and model versions.
Figure~\ref{fig:capability_shared48} shows positive rank associations in all
six scopes, strongest for Joint. Spearman correlations are $\rho=0.904$,
$0.794$, $0.701$, $0.596$, $0.743$, and $0.756$ for Joint, Q, K, V, O, and
FFN, respectively.
\setcounter{topnumber}{1}%
\begin{figure}[t]
  \centering
  \scalebox{1}[0.92]{\includegraphics[width=\linewidth]{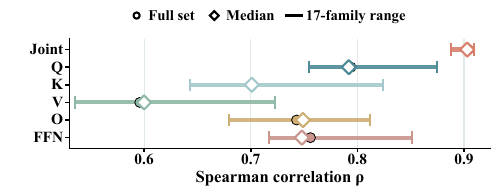}}
  \caption{\textbf{Capability Correlations after Removing One Model Family at a Time.}
  Horizontal ranges show the 17 correlations obtained after deleting one complete
  model family. Filled circles show the full-set $\rho$, open diamonds show the
  median after removal.}
  \label{fig:capability-lofo}
\end{figure}
Ties receive midranks. Fused QKV projections are split into their native
Q/K/V output-row ranges before scoring. Appendix~\ref{app:capability-controls} and
Figure~\ref{fig:capability-lofo} report label, representation, and
leave-one-family-out controls supporting
the association across model families.

\textbf{Removing One Model Family at a Time.}
We remove each of the 17 model families in turn and recompute the original
Spearman correlation using all remaining models. Figure~\ref{fig:capability-lofo} reports the full-set value, the range and median across the 17 removals, and
the largest absolute change.
All six correlations remain positive after every family removal. Joint changes
the least, ranging from $0.888$ to $0.909$, with a maximum absolute change of
$0.016$.

\textbf{Stability under Complete-Benchmark Exposure.}
To examine whether benchmark-specific exposure alters the relative rankings of model ability and \dvlkm{}, we conduct a controlled experiment in which models are trained on the full test split of each benchmark.
Specifically, we expose six models to all five benchmarks.
For each model–benchmark pair, we independently fine-tune the original checkpoint on the benchmark's reference answers and recompute its benchmark and \dvlkm{} scores, while keeping the structural configuration unchanged.
We then compare the updated rankings against the other models' unchanged baselines.
We consider a \dvlkm{} ranking stable if it shifts by at most one position, and define rank advance as the original rank minus the updated rank, with positive values indicating improvement.
As shown in Figure~\ref{fig:fullbench_rank_stability}, benchmark exposure improves ability rankings while leaving \dvlkm{} rankings largely unchanged, highlighting the stability of \dvlkm{}.

\vspace{3.3pt}%
\section{Conclusion}
We introduce \dvl{}, mapping LLMs' parameter columns to vocabulary-space readouts, connecting parameter interpretation with targeted interventions. Statistical tests support structure and local predictive accuracy of these readouts, while decoded cases illustrate behavioral effects. \mbox{\dvlkm{}} summarizes functional organization and tracks capability across models. Its ranks remain stable under complete-benchmark exposure, as benchmark performance improves. These results connect parameter analysis with model-level comparison through a shared vocabulary-space representation. Vocabulary-space readouts offer one view of parameter organization, supporting examining parameters and comparing functional organization across models.
\subsection*{LLM Usage Statement}
Generative AI tools were used to polish the grammar and wording of the manuscript. The authors reviewed the final text and take responsibility for its content.

{\footnotesize
\bibliography{refs}
\bibliographystyle{ieeenat_fullname}

}

\clearpage
\appendix
\etocdepthtag.toc{appendix}
\suppFrontMatter
\section{Experimental Protocols}
\label{app:protocols}

\subsection{Intervention Setup}

The intervention experiments cover 21 models, ten factual concepts, six
scopes, three operations, and eight intervention fractions:
$0.05\%$, $0.1\%$, $0.2\%$, $0.5\%$, $1\%$, $2\%$, $3\%$, and $5\%$.
The six scopes are Joint, which combines Q+K+V+O+FFN, and Q, K, V, O, and FFN
separately. For Joint, the requested fraction is selected independently within
Q, K, V, O, and FFN, and the five selected sets are then edited together.
The main decoded measure includes only clean continuations that contain the
original gold answer.

\subsection{Capability Analysis}

The analysis uses a fixed set of 48 models and the joint-layer output-metric
matrix computed from 2,000 prompts at every eligible layer, with rank 128 and
$k=32$. Joint, Q, K, V, O, and FFN are all evaluated on the same 48 models.

Six Phi-3 and InternLM2 models use a fused QKV projection. For these models,
we split both the frozen operator and the model weights according to the
architecture's original Q, K, and V layout before computing separate scores.
For Phi-3, Q, K, and V occupy contiguous row ranges determined by the numbers
of query and key/value heads. For InternLM2, rows are grouped by key/value
head, with the corresponding query group followed by its K and V rows. In
both cases, the Q, K, and V parts are disjoint and together cover the full
output dimension.

The six scopes therefore give $48\times6=288$ observed values. Correlation
coefficients are recomputed directly from the displayed descending ranks.
The fitted lines summarize the relation within this fixed set of models.

\subsection{Capability Score and Fixed Model Set}
\label{app:capability-score}

We use the same capability score $y_j$ for all 48 models. The Open LLM
Leaderboard v2 evaluates BBH, MATH Level 5, GPQA, MuSR, and MMLU-Pro
under fixed task-specific settings
~\citep{huggingfaceollbv2,suzgun2022bbh,hendrycks2021math,
rein2023gpqa,sprague2024musr,wang2024mmlupro}.
Let $s_{j,t}$ denote model $j$'s percentage score on task $t$. We define the
capability score as the unweighted mean of the five tasks:
\begin{equation}
y_j=\frac{1}{5}\left(
s_{j,\mathrm{BBH}}+s_{j,\mathrm{MATH5}}+
s_{j,\mathrm{GPQA}}+s_{j,\mathrm{MuSR}}+s_{j,\mathrm{MMLUPro}}
\right).
\label{eq:capability-score}
\end{equation}
The fixed settings are three-shot normalized BBH accuracy, four-shot exact match on MATH
Level 5, zero-shot normalized GPQA accuracy, zero-shot normalized MuSR
accuracy across its three subtasks, and five-shot MMLU-Pro accuracy. For
post-trained models, we use the released chat template and present few-shot
examples as multiturn conversations, following the official Open LLM Leaderboard v2 protocol.
Higher values are better for all five tasks, and all five scores are available
for every model.
For each scope $\mathcal S$, the model-level \dvlkm{} score averages all
eligible matrices as defined in Eq.~\ref{eq:model_dvl_kmeans}. We rank
the 48 capability scores and the corresponding 48 \dvlkm{} scores in
descending order. Rank 1 denotes the largest value, and ties use the standard
midrank. The reported Spearman coefficient is the Pearson correlation between
these two 48-entry rank vectors, following
Equations~\ref{eq:dvl_capability_ranks}--\ref{eq:dvl_capability_correlation}.
Joint, Q, K, V, O, and FFN therefore use the same capability score, the same
48 models, and all 48 paired ranks. No models are removed for individual
modules. The set includes instruction-tuned, chat, and hybrid reasoning
models. For the six fused-QKV models, the architecture-specific splits
described above provide separate Q, K, and V measurements.
Table~\ref{tab:capability-LLMs} lists all models, while their released
versions, architectures, and capability scores are provided in the
supplementary material.

\begingroup
\dvlTableSetup
\hypersetup{urlcolor=black,linkcolor=black}
\setlength{\tabcolsep}{4pt}
\renewcommand{\arraystretch}{1.06}
\begin{longtable}{@{}>{\raggedleft\arraybackslash}p{0.035\linewidth}>{\raggedright\arraybackslash}p{0.80\linewidth}>{\raggedright\arraybackslash}p{\dimexpr0.165\linewidth-4\tabcolsep\relax}@{}}
\caption{\textbf{The Fixed Set of 48 LLMs Used with the Uniform Five-Task Capability Definition.} The table lists the LLM and its type. The released version, architecture, and capability score of every LLM are recorded with the supplementary material.}
\label{tab:capability-LLMs}\\
\toprule
\dvlHead{\#} & \dvlHead{LLM repository} & \dvlHead{Type} \\
\midrule
\endfirsthead
\multicolumn{3}{l}{\tablename\ \thetable\ (continued)}\\
\toprule
\dvlHead{\#} & \dvlHead{LLM repository} & \dvlHead{Type} \\
\midrule
\endhead
\midrule
\multicolumn{3}{r}{Continued on next page}\\
\endfoot
\bottomrule
\endlastfoot
1 & \href{https://huggingface.co/tiiuae/Falcon3-1B-Instruct}{tiiuae/Falcon3-1B-Instruct}~\cite{tii2024falcon3} & instruct \\
2 & \href{https://huggingface.co/tiiuae/Falcon3-7B-Instruct}{tiiuae/Falcon3-7B-Instruct}~\cite{tii2024falcon3} & instruct \\
3 & \href{https://huggingface.co/NousResearch/Hermes-2-Pro-Llama-3-8B}{NousResearch/Hermes-2-Pro-Llama-3-8B}~\cite{grattafiori2024llama3} & chat \\
4 & \href{https://huggingface.co/NousResearch/Hermes-3-Llama-3.1-8B}{NousResearch/Hermes-3-Llama-3.1-8B}~\cite{teknium2024hermes3} & chat \\
5 & \href{https://huggingface.co/unsloth/Llama-3.1-8B-Instruct}{unsloth/Llama-3.1-8B-Instruct}~\cite{grattafiori2024llama3} & instruct \\
6 & \href{https://huggingface.co/meta-llama/Llama-3.2-1B-Instruct}{meta-llama/Llama-3.2-1B-Instruct}~\cite{grattafiori2024llama3} & instruct \\
7 & \href{https://huggingface.co/unsloth/Llama-3.2-3B-Instruct}{unsloth/Llama-3.2-3B-Instruct}~\cite{grattafiori2024llama3} & instruct \\
8 & \href{https://huggingface.co/mistralai/Ministral-8B-Instruct-2410}{mistralai/Ministral-8B-Instruct-2410}~\cite{mistral2024ministral} & instruct \\
9 & \href{https://huggingface.co/mistralai/Mistral-7B-Instruct-v0.2}{mistralai/Mistral-7B-Instruct-v0.2}~\cite{jiang2023mistral} & instruct \\
10 & \href{https://huggingface.co/mistralai/Mistral-7B-Instruct-v0.3}{mistralai/Mistral-7B-Instruct-v0.3}~\cite{jiang2023mistral} & instruct \\
11 & \href{https://huggingface.co/NousResearch/Nous-Hermes-2-Mistral-7B-DPO}{NousResearch/Nous-Hermes-2-Mistral-7B-DPO}~\cite{jiang2023mistral} & chat \\
12 & \href{https://huggingface.co/teknium/OpenHermes-2.5-Mistral-7B}{teknium/OpenHermes-2.5-Mistral-7B}~\cite{jiang2023mistral} & chat \\
13 & \href{https://huggingface.co/microsoft/Phi-3-medium-4k-instruct}{microsoft/Phi-3-medium-4k-instruct}~\cite{abdin2024phi3} & instruct \\
14 & \href{https://huggingface.co/microsoft/Phi-3-mini-4k-instruct}{microsoft/Phi-3-mini-4k-instruct}~\cite{abdin2024phi3} & instruct \\
15 & \href{https://huggingface.co/microsoft/Phi-3.5-mini-instruct}{microsoft/Phi-3.5-mini-instruct}~\cite{abdin2024phi3} & instruct \\
16 & \href{https://huggingface.co/Qwen/Qwen1.5-0.5B-Chat}{Qwen/Qwen1.5-0.5B-Chat}~\cite{bai2023qwen} & chat \\
17 & \href{https://huggingface.co/Qwen/Qwen1.5-1.8B-Chat}{Qwen/Qwen1.5-1.8B-Chat}~\cite{bai2023qwen} & chat \\
18 & \href{https://huggingface.co/Qwen/Qwen1.5-14B-Chat}{Qwen/Qwen1.5-14B-Chat}~\cite{bai2023qwen} & chat \\
19 & \href{https://huggingface.co/Qwen/Qwen1.5-4B-Chat}{Qwen/Qwen1.5-4B-Chat}~\cite{bai2023qwen} & chat \\
20 & \href{https://huggingface.co/Qwen/Qwen2-0.5B-Instruct}{Qwen/Qwen2-0.5B-Instruct}~\cite{yang2024qwen2} & instruct \\
21 & \href{https://huggingface.co/Qwen/Qwen2-1.5B-Instruct}{Qwen/Qwen2-1.5B-Instruct}~\cite{yang2024qwen2} & instruct \\
22 & \href{https://huggingface.co/Qwen/Qwen2-7B-Instruct}{Qwen/Qwen2-7B-Instruct}~\cite{yang2024qwen2} & instruct \\
23 & \href{https://huggingface.co/Qwen/Qwen2.5-0.5B-Instruct}{Qwen/Qwen2.5-0.5B-Instruct}~\cite{qwen2025qwen25} & instruct \\
24 & \href{https://huggingface.co/Qwen/Qwen2.5-14B-Instruct}{Qwen/Qwen2.5-14B-Instruct}~\cite{qwen2025qwen25} & instruct \\
25 & \href{https://huggingface.co/Qwen/Qwen2.5-Coder-7B-Instruct}{Qwen/Qwen2.5-Coder-7B-Instruct}~\cite{hui2024qwen25coder} & instruct \\
26 & \href{https://huggingface.co/Qwen/Qwen2.5-Math-1.5B-Instruct}{Qwen/Qwen2.5-Math-1.5B-Instruct}~\cite{yang2024qwen25math} & instruct \\
27 & \href{https://huggingface.co/Qwen/Qwen2.5-Math-7B-Instruct}{Qwen/Qwen2.5-Math-7B-Instruct}~\cite{yang2024qwen25math} & instruct \\
28 & \href{https://huggingface.co/Qwen/Qwen3-0.6B}{Qwen/Qwen3-0.6B}~\cite{yang2025qwen3} & hybrid \\
29 & \href{https://huggingface.co/Qwen/Qwen3-1.7B}{Qwen/Qwen3-1.7B}~\cite{yang2025qwen3} & hybrid \\
30 & \href{https://huggingface.co/HuggingFaceTB/SmolLM-1.7B-Instruct}{HuggingFaceTB/SmolLM-1.7B-Instruct}~\cite{allal2024smollm} & instruct \\
31 & \href{https://huggingface.co/HuggingFaceTB/SmolLM-135M-Instruct}{HuggingFaceTB/SmolLM-135M-Instruct}~\cite{allal2024smollm} & instruct \\
32 & \href{https://huggingface.co/HuggingFaceTB/SmolLM-360M-Instruct}{HuggingFaceTB/SmolLM-360M-Instruct}~\cite{allal2024smollm} & instruct \\
33 & \href{https://huggingface.co/HuggingFaceTB/SmolLM2-1.7B-Instruct}{HuggingFaceTB/SmolLM2-1.7B-Instruct}~\cite{allal2025smollm2} & instruct \\
34 & \href{https://huggingface.co/berkeley-nest/Starling-LM-7B-alpha}{berkeley-nest/Starling-LM-7B-alpha}~\cite{zhu2024starling} & chat \\
35 & \href{https://huggingface.co/01-ai/Yi-1.5-6B-Chat}{01-ai/Yi-1.5-6B-Chat}~\cite{young2024yi} & chat \\
36 & \href{https://huggingface.co/01-ai/Yi-1.5-9B-Chat}{01-ai/Yi-1.5-9B-Chat}~\cite{young2024yi} & chat \\
37 & \href{https://huggingface.co/cognitivecomputations/dolphin-2.9.3-mistral-7B-32k}{cognitivecomputations/dolphin-2.9.3-mistral-7B-32k}~\cite{jiang2023mistral} & chat \\
38 & \href{https://huggingface.co/ibm-granite/granite-3.0-2b-instruct}{ibm-granite/granite-3.0-2b-instruct}~\cite{granite2024granite3} & instruct \\
39 & \href{https://huggingface.co/ibm-granite/granite-3.0-8b-instruct}{ibm-granite/granite-3.0-8b-instruct}~\cite{granite2024granite3} & instruct \\
40 & \href{https://huggingface.co/ibm-granite/granite-3.1-8b-instruct}{ibm-granite/granite-3.1-8b-instruct}~\cite{granite2024granite3} & instruct \\
41 & \href{https://huggingface.co/ibm-granite/granite-3.2-8b-instruct}{ibm-granite/granite-3.2-8b-instruct}~\cite{granite2024granite3} & instruct \\
42 & \href{https://huggingface.co/ibm-granite/granite-3.3-8b-instruct}{ibm-granite/granite-3.3-8b-instruct}~\cite{granite2024granite3} & instruct \\
43 & \href{https://huggingface.co/internlm/internlm2-chat-1_8b}{internlm/internlm2-chat-1\_8b}~\cite{cai2024internlm2} & chat \\
44 & \href{https://huggingface.co/internlm/internlm2_5-1_8b-chat}{internlm/internlm2\_5-1\_8b-chat}~\cite{cai2024internlm2} & chat \\
45 & \href{https://huggingface.co/internlm/internlm2_5-7b-chat}{internlm/internlm2\_5-7b-chat}~\cite{cai2024internlm2} & chat \\
46 & \href{https://huggingface.co/openchat/openchat-3.5-0106}{openchat/openchat-3.5-0106}~\cite{wang2024openchat} & chat \\
47 & \href{https://huggingface.co/stabilityai/stablelm-2-1_6b-chat}{stabilityai/stablelm-2-1\_6b-chat}~\cite{bellagente2024stablelm2} & chat \\
48 & \href{https://huggingface.co/stabilityai/stablelm-2-zephyr-1_6b}{stabilityai/stablelm-2-zephyr-1\_6b}~\cite{bellagente2024stablelm2,tunstall2023zephyr} & chat \\
\end{longtable}
\endgroup

\subsection{Capability Controls}
\label{app:capability-controls}

All controls use the same joint-layer output-metric representation, the same
48 models and capability ranks as Figure~\ref{fig:capability_shared48}, the
same layers and matrices, $k=32$, seed 0, 12 k-means iterations, and the same
architecture-specific Q/K/V splits. In every table, the \dvlkm{} value is
exactly the value reported in Figure~\ref{fig:capability_shared48}.

\subsubsection*{Random Labels with Matched Cluster Sizes.}

For each matrix, we permute the \dvlkm{} labels 128 times while keeping the
sizes of the occupied clusters unchanged. Table~\ref{tab:capability-random-label}
compares the original coefficient with the mean and standard deviation (SD) of
these random-label results. The one-sided conditional randomization $p$-value
is
\begin{equation}
p_{\mathrm{cond}}
=
\frac{
1+\sum_{b=1}^{128}
\mathbf{1}\!\left\{
\rho^{(b)}_{\mathrm{rand}}
\geq
\rho_{\mathrm{DVL}}
\right\}
}{129},
\label{eq:capability-random-label-p}
\end{equation}
with a minimum possible value of $1/129\approx0.008$.

\begin{table}[!htp]
\centering
\dvlTableSetup
\setlength{\tabcolsep}{4pt}
\renewcommand{\arraystretch}{1.05}
\caption{\textbf{Random-Label Control with Matched Cluster Sizes on the Fixed
48-Model Set.} All six \dvlkm{} correlations are larger than those from all
128 random-label permutations, giving the minimum possible conditional
randomization value.}
\label{tab:capability-random-label}
\begin{tabularx}{\linewidth}{>{\raggedright\arraybackslash}p{.08\linewidth}*{4}{>{\raggedleft\arraybackslash}X}}
\toprule
\dvlHead{Scope} & \dvlHead{\shortstack{\dvlkm{}\\$\rho$}} & \dvlHead{Random mean $\rho$} & \dvlHead{Random SD} & \dvlHead{$p_{\mathrm{cond}}$} \\
\midrule
\rowcolor{dvltint}\dvlScope{modulejoint}{Joint} & 0.904 & 0.216 & 0.005 & 0.008 \\
\dvlScope{moduleq}{Q}     & 0.794 & 0.217 & 0.013 & 0.008 \\
\dvlScope{modulek}{K}     & 0.701 & 0.214 & 0.015 & 0.008 \\
\dvlScope{modulev}{V}     & 0.596 & 0.180 & 0.022 & 0.008 \\
\dvlScope{moduleo}{O}     & 0.743 & 0.149 & 0.010 & 0.008 \\
\dvlScope{moduleffn}{FFN}   & 0.756 & 0.110 & 0.012 & 0.008 \\
\bottomrule
\end{tabularx}
\end{table}

These results show that the observed capability correlations are not
reproduced by random cluster assignments with the same cluster sizes.
However, the model set and capability ranks remain fixed, so these values
should not be read as significance tests for a broader model population or
for unseen model families.

\subsubsection*{Raw-Weight Baseline.}

The raw-weight baseline clusters directions directly in the original
weight-column space. It uses the same $k=32$, seed 0, 12 iterations, models,
layers, and matrices, and evaluates the same raw-space relative sum-of-squared-errors (SSE) score.
Table~\ref{tab:capability-alternative-baselines} reports
$\Delta\rho=\rho_{\mathrm{DVL}}-\rho_{\mathrm{baseline}}$, together with a
95\% interval from 2,000 model-family-cluster bootstrap samples.

\begin{table}[!htp]
\centering
\dvlTableSetup
\setlength{\tabcolsep}{4pt}
\renewcommand{\arraystretch}{1.05}
\caption{\textbf{Raw-Weight Baseline on the Fixed 48-Model Set.}
Positive $\Delta\rho$ means that the \dvlkm{} correlation is higher.
The last column gives the model-family-cluster bootstrap interval for the
difference.}
\label{tab:capability-alternative-baselines}
\begin{tabularx}{\linewidth}{>{\raggedright\arraybackslash}p{.08\linewidth}*{4}{>{\raggedleft\arraybackslash}X}}
\toprule
\dvlHead{Scope} & \dvlHead{\shortstack{\dvlkm{}\\$\rho$}} & \dvlHead{Baseline $\rho\,\uparrow$} & \dvlHead{$\Delta\rho$} & \dvlHead{Bootstrap 95\% CI} \\
\midrule 
\rowcolor{dvltint}\dvlScope{modulejoint}{Joint} & 0.904 & 0.800 & \dvlGood{$+0.104$} & $[-0.023,\ 0.311]$ \\
\dvlScope{moduleq}{Q}     & 0.794 & 0.653 & \dvlGood{$+0.141$} & $[-0.073,\ 0.457]$ \\
\dvlScope{modulek}{K}     & 0.701 & 0.519 & \dvlGood{$+0.181$} & $[-0.083,\ 0.532]$ \\
\dvlScope{modulev}{V}     & 0.596 & 0.593 & \dvlGood{$+0.003$} & $[-0.293,\ 0.331]$ \\
\dvlScope{moduleo}{O}     & 0.743 & 0.689 & \dvlGood{$+0.054$} & $[-0.023,\ 0.210]$ \\
\dvlScope{moduleffn}{FFN}   & 0.756 & 0.647 & \dvlGood{$+0.109$} & $[-0.034,\ 0.304]$ \\
\bottomrule
\end{tabularx}
\end{table}

The \dvlkm{} correlation is higher in all six comparisons. However, every
bootstrap interval includes zero, so this experiment supports a consistent
direction in the point estimates rather than a clear difference under the
bootstrap analysis.

\subsection{Complete-Benchmark Exposure}
\label{app:benchmark-exposure}

\subsubsection*{Rank Comparison.}

For each benchmark, the exposed checkpoint replaces only its own baseline
score; all other models retain their original scores. We apply this procedure
separately to benchmark performance and \dvlkm{}. Higher scores receive
better ranks, rank 1 is the highest, and ties receive midranks. A \dvlkm{} rank change with absolute value at most one is
classified as stable. Models already at benchmark rank 1 have no room for
further rank advance. Figure~\ref{fig:fullbench_rank_stability} displays the
paired advances for all six exposed models.

\section{Statistical Boundaries and Structural Inspection}
\label{app:statistical_boundaries}

\subsection{Matched-Null Coverage and Variation}
\label{app:matched_null_breadth}

We use the same concept-mean comparison across 21 models from 13 families.
For each model, we evaluate three fixed depth fractions and five modules:
Q, K, V, O, and FFN. This gives 315 model-depth-module cells. In each cell,
12 concept-mean \dvl{} directions are compared with 192 matched Gaussian
directions, resulting in 3,780 concept-mean directions, 30,240 prompt
references, and 60,480 Gaussian directions. Table~\ref{tab:a79-models-21}
reports the model-level results, while
Tables~\ref{tab:a79-modules}--\ref{tab:a79-lofo} summarize the results by
module, family, and after removing one family at a time. All tables use the
same model names and order as the intervention figures.

Figure~\ref{fig:matched_null_21models} shows that every model has a positive
mean difference in both panels, although the amount of variation differs
across models. The overall mean also stays clearly above zero for both excess
kurtosis and Top-1\% $\ell_2$-energy. Thus, the positive pattern is shared
across models, while its size varies across models, depths, and modules.

The module-level results show a similar pattern. Q, K, O, and FFN have
positive model-level means for both measures in all 21 models. V shows a
weaker pattern, with positive means in 19 of the 21 models for each measure.
These counts are based on model-level module averages. The intervals in
Figure~\ref{fig:matched_null_21models} and
Table~\ref{tab:a79-models-21} show the variation across the 15
depth-by-module cells within each model.

We also repeat the analysis after removing each model family in turn. Across
all 13 such settings, the mean excess-kurtosis difference remains between
$2.195$ and $2.472$, while the Top-1\% $\ell_2$-energy difference remains
between $0.0403$ and $0.0432$. Both measures stay positive for every remaining
model in every setting. The positive pattern persists after removing any
one model family.

\par\noindent\begin{minipage}{\textwidth}
\centering
\dvlTableSetup
\captionsetup{hypcap=false}
\captionof{table}{\textbf{Model-Level Matched-Null Differences across 21 Models.}
Each entry reports the mean $\Delta$ [95\% CI]. For each model, the interval is
computed by resampling its 15 depth-by-module cells. The last column shows how
many of the 15 cells have a positive difference in excess kurtosis and
Top-1\% $\ell_2$-energy, respectively. The final row instead resamples the 21
model means, with counts taken over models. Model names and ordering match the
intervention figures; DS-R1 denotes DeepSeek-R1-Distill-Qwen.}
\label{tab:a79-models-21}
\setlength{\tabcolsep}{3.8pt}
\renewcommand{\arraystretch}{1.05}
\begin{tabularx}{\textwidth}{>{\hsize=.9\hsize\raggedright\arraybackslash}X*{2}{>{\hsize=1.05\hsize\raggedleft\arraybackslash}X}}
\toprule
\dvlHead{Model} & \dvlHead{\shortstack[r]{Excess kurtosis\\$\Delta\,\uparrow$ [95\% CI]}} & \dvlHead{\shortstack[r]{Top-1\% $\ell_2$-energy\\$\Delta\,\uparrow$ [95\% CI]}} \\
\midrule
\dvlModel{deepseek}{DS-R1-1.5B}                    & $1.159\ [0.516,1.870]$ & $0.0356\ [0.0167,0.0565]$ \\
\dvlModel{deepseek}{DS-R1-7B}                      & $1.502\ [0.772,2.300]$ & $0.0308\ [0.0173,0.0453]$ \\
\dvlModel{falcon}{Falcon3-3B}                    & $2.048\ [0.694,4.183]$ & $0.0373\ [0.0212,0.0555]$ \\
\dvlModel{falcon}{Falcon3-7B}                    & $1.453\ [0.776,2.262]$ & $0.0378\ [0.0229,0.0566]$ \\
\dvlModel{ibm}{Granite-3.0-2B}                    & $2.540\ [1.401,3.888]$ & $0.0527\ [0.0331,0.0735]$ \\
\dvlModel{ibm}{Granite-3.3-8B}                    & $5.391\ [1.714,10.192]$ & $0.0647\ [0.0353,0.0996]$  \\
\dvlModel{meta}{Llama-3.2-1B}                   & $3.375\ [1.173,6.308]$ & $0.0540\ [0.0299,0.0796]$ \\
\dvlModel{meta}{Llama-3.2-3B}                   & $3.107\ [1.372,5.190]$ & $0.0533\ [0.0323,0.0747]$ \\
\dvlModel{meta}{Llama-3.1-8B}                   & $4.624\ [1.678,8.645]$ & $0.0603\ [0.0324,0.0908]$  \\
\dvlModel{mistral}{Mistral-7B}                    & $3.263\ [1.224,5.771]$ & $0.0498\ [0.0269,0.0741]$ \\
\dvlModel{ai2}{OLMo-2-7B}                      & $0.630\ [0.284,1.054]$ & $0.0205\ [0.0104,0.0324]$ \\
\dvlModel{qwen}{Qwen1.5-14B}                   & $2.858\ [0.594,6.522]$ & $0.0261\ [0.0138,0.0412]$ \\
\dvlModel{qwen}{Qwen2.5-0.5B}                  & $1.488\ [0.787,2.316]$ & $0.0393\ [0.0250,0.0550]$ \\
\dvlModel{qwen}{Qwen2.5-1.5B}                  & $2.678\ [1.230,4.815]$ & $0.0570\ [0.0351,0.0818]$ \\
\dvlModel{qwen}{Qwen2.5-3B}                    & $1.905\ [0.391,4.551]$ & $0.0282\ [0.0141,0.0467]$ \\
\dvlModel{qwen}{Qwen2.5-7B}                    & $1.950\ [0.665,3.504]$ & $0.0376\ [0.0172,0.0629]$ \\
\dvlModel{qwen}{Qwen2.5-14B}                   & $2.534\ [1.128,4.283]$ & $0.0460\ [0.0265,0.0668]$  \\
\dvlModel{qwen}{Qwen3-1.7B}                    & $2.043\ [0.966,3.283]$ & $0.0462\ [0.0260,0.0687]$ \\
\dvlModel{upstage}{SOLAR-10.7B}                   & $2.261\ [0.857,4.252]$ & $0.0377\ [0.0201,0.0579]$ \\
\dvlModel{stability}{StableLM-2-1.6B}                 & $1.693\ [0.965,2.624]$ & $0.0413\ [0.0283,0.0542]$ \\
\dvlModel{yi}{Yi-1.5-9B}                     & $1.128\ [0.561,1.735]$ & $0.0278\ [0.0146,0.0428]$ \\
\midrule
\rowcolor{dvltint}\textbf{Mean} & $\mathbf{2.363\ [1.911,2.869]}$ & $\mathbf{0.0421\ [0.0371,0.0472]}$  \\
\bottomrule
\end{tabularx}
\end{minipage}\par

\par\noindent\begin{minipage}{\textwidth}
\centering
\dvlTableSetup\fontsize{7.5pt}{8.4pt}\selectfont
\captionsetup{hypcap=false}
\captionof{table}{\textbf{Matched-Null Differences across All 21 Models and by Module.}
Each estimate is the unweighted mean of the 21 model-level values. The two
intervals are obtained by resampling models and model families, respectively.
The two $p$-values are one-sided exact sign-test and Wilcoxon-test results
computed across models.}
\label{tab:a79-modules}
\setlength{\tabcolsep}{1.5pt}
\renewcommand{\arraystretch}{1.05}
\begin{tabularx}{\textwidth}{l>{\raggedright\arraybackslash}Xrrrrrr}
\toprule
\dvlHead{Scope} & \dvlHead{Metric} & \dvlHead{Mean $\Delta\,\uparrow$} & \dvlHead{\shortstack{95\% CI\\(model)}} & \dvlHead{\shortstack{95\% CI\\(family)}} & \dvlHead{\shortstack{Positive\\models}} & \dvlHead{Sign $p\,\downarrow$} & \dvlHead{\shortstack{Wilcoxon\\$p\,\downarrow$}} \\
\midrule
\rowcolor{dvltint}\dvlScope{modulejoint}{All modules} & Excess kurtosis & 2.363 & [1.911, 2.869] & [1.858, 2.958] & 21/21 & $4.768\times 10^{-7}$ & $4.768\times 10^{-7}$ \\
\rowcolor{dvltint}\dvlScope{modulejoint}{All modules} & Top-1\% $\ell_2$-energy & 0.0421 & [0.0371, 0.0472] & [0.0363, 0.0478] & 21/21 & $4.768\times 10^{-7}$ & $4.768\times 10^{-7}$ \\
\addlinespace[2pt]
\dvlScope{moduleq}{Q} & Excess kurtosis & 1.577 & [0.621, 2.823] & [0.427, 2.986] & 21/21 & $4.768\times 10^{-7}$ & $4.768\times 10^{-7}$ \\
\dvlScope{moduleq}{Q} & Top-1\% $\ell_2$-energy & 0.0264 & [0.0178, 0.0373] & [0.0148, 0.0380] & 21/21 & $4.768\times 10^{-7}$ & $4.768\times 10^{-7}$ \\
\dvlInner{1-8}
\dvlScope{modulek}{K} & Excess kurtosis & 2.447 & [1.510, 3.664] & [1.502, 3.668] & 21/21 & $4.768\times 10^{-7}$ & $4.768\times 10^{-7}$ \\
\dvlScope{modulek}{K} & Top-1\% $\ell_2$-energy & 0.0443 & [0.0330, 0.0573] & [0.0326, 0.0561] & 21/21 & $4.768\times 10^{-7}$ & $4.768\times 10^{-7}$ \\
\dvlInner{1-8}
\dvlScope{modulev}{V} & Excess kurtosis & 0.313 & [0.167, 0.504] & [0.171, 0.545] & 19/21 & $1.106\times 10^{-4}$ & $2.384\times 10^{-6}$ \\
\dvlScope{modulev}{V} & Top-1\% $\ell_2$-energy & 0.0092 & [0.0054, 0.0139] & [0.0055, 0.0146] & 19/21 & $1.106\times 10^{-4}$ & $3.338\times 10^{-6}$ \\
\dvlInner{1-8}
\dvlScope{moduleo}{O} & Excess kurtosis & 2.043 & [1.595, 2.529] & [1.642, 2.564] & 21/21 & $4.768\times 10^{-7}$ & $4.768\times 10^{-7}$ \\
\dvlScope{moduleo}{O} & Top-1\% $\ell_2$-energy & 0.0571 & [0.0466, 0.0682] & [0.0474, 0.0699] & 21/21 & $4.768\times 10^{-7}$ & $4.768\times 10^{-7}$ \\
\dvlInner{1-8}
\dvlScope{moduleffn}{FFN} & Excess kurtosis & 5.437 & [3.945, 7.103] & [3.756, 7.505] & 21/21 & $4.768\times 10^{-7}$ & $4.768\times 10^{-7}$ \\
\dvlScope{moduleffn}{FFN} & Top-1\% $\ell_2$-energy & 0.0734 & [0.0613, 0.0858] & [0.0580, 0.0885] & 21/21 & $4.768\times 10^{-7}$ & $4.768\times 10^{-7}$ \\
\bottomrule
\end{tabularx}
\end{minipage}\par

\par\noindent\begin{minipage}{\textwidth}
\centering
\dvlTableSetup
\captionsetup{hypcap=false}
\captionof{table}{\textbf{Matched-Null Differences by Model Family.} Family values are descriptive unweighted means of the model-level estimates in Table~\ref{tab:a79-models-21}; DS-R1 denotes DeepSeek-R1-Distill-Qwen.}
\label{tab:a79-families-13}
\setlength{\tabcolsep}{3.5pt}
\renewcommand{\arraystretch}{1.05}
\begin{tabular*}{\textwidth}{@{\extracolsep{\fill}}>{\raggedright\arraybackslash}p{0.15\textwidth}r>{\raggedright\arraybackslash}p{0.32\textwidth}rr@{}}
\toprule
\dvlHead{Family} & \dvlHead{$N$} & \dvlHead{Models} & \dvlHead{\shortstack{Excess kurtosis\\$\Delta\,\uparrow$}} & \dvlHead{\shortstack{Top-1\% $\ell_2$-energy\\$\Delta\,\uparrow$}} \\
\midrule
\dvlModel{deepseek}{DS-R1} & 2 & DS-R1-1.5B; DS-R1-7B & 1.331 & 0.0332 \\
\dvlModel{falcon}{Falcon3} & 2 & Falcon3-3B; Falcon3-7B & 1.751 & 0.0376 \\
\dvlModel{ibm}{Granite} & 2 & Granite-3.0-2B; Granite-3.3-8B & 3.965 & 0.0587 \\
\dvlModel{meta}{Llama-3.1} & 1 & Llama-3.1-8B & 4.624 & 0.0603 \\
\dvlModel{meta}{Llama-3.2} & 2 & Llama-3.2-1B; Llama-3.2-3B & 3.241 & 0.0536 \\
\dvlModel{mistral}{Mistral} & 1 & Mistral-7B & 3.263 & 0.0498 \\
\dvlModel{ai2}{OLMo-2} & 1 & OLMo-2-7B & 0.630 & 0.0205 \\
\dvlModel{qwen}{Qwen1.5} & 1 & Qwen1.5-14B & 2.858 & 0.0261 \\
\dvlModel{qwen}{Qwen2.5} & 5 & Qwen2.5-0.5B; Qwen2.5-1.5B; Qwen2.5-3B; Qwen2.5-7B; Qwen2.5-14B & 2.111 & 0.0416 \\
\dvlModel{qwen}{Qwen3} & 1 & Qwen3-1.7B & 2.043 & 0.0462 \\
\dvlModel{upstage}{SOLAR} & 1 & SOLAR-10.7B & 2.261 & 0.0377 \\
\dvlModel{stability}{StableLM} & 1 & StableLM-2-1.6B & 1.693 & 0.0413 \\
\dvlModel{yi}{Yi-1.5} & 1 & Yi-1.5-9B & 1.128 & 0.0278 \\
\bottomrule
\end{tabular*}
\end{minipage}\par

\begin{figure}[!t]
  \centering
  {\fontsize{8}{9}\selectfont\input{paper_figures/FigMN_composite_labels}}
  \caption{\textbf{Matched-Null Differences for the 21 Models.}
  Circles are model means with cell-bootstrap 95\% intervals over the 15 depth-by-module cells; the diamond is the mean over models with a model-bootstrap 95\% interval; the gold dashed line marks zero. Both differences are positive for every model. Model codes as in Figure~\ref{fig:intervention_all_models}.}
  \label{fig:matched_null_21models}
\end{figure}

\par\noindent\begin{minipage}{\textwidth}
\centering
\dvlTableSetup
\captionsetup{hypcap=false}
\captionof{table}{\textbf{Matched-Null Differences after Removing One Family at a Time.} Every row removes one complete family and recomputes the unweighted model mean over the 12 families that remain. Positive models counts the retained models with a positive difference, which is the same for both metrics in every row; DS-R1 denotes DeepSeek-R1-Distill-Qwen.}
\label{tab:a79-lofo}
\setlength{\tabcolsep}{3.5pt}
\renewcommand{\arraystretch}{1.05}
\begin{tabular*}{\textwidth}{@{\extracolsep{\fill}}lrrrr@{}}
\toprule
\dvlHead{Omitted family} & \dvlHead{$N$} & \dvlHead{Excess kurtosis $\Delta\,\uparrow$} & \dvlHead{Top-1\% $\ell_2$-energy $\Delta\,\uparrow$} & \dvlHead{Positive models} \\
\midrule
\dvlModel{deepseek}{DS-R1} & 19 & 2.472 & 0.0430 & 19/19 \\
\dvlModel{falcon}{Falcon3} & 19 & 2.428 & 0.0426 & 19/19 \\
\dvlModel{ibm}{Granite} & 19 & 2.195 & 0.0403 & 19/19 \\
\dvlModel{meta}{Llama-3.1} & 20 & 2.250 & 0.0412 & 20/20 \\
\dvlModel{meta}{Llama-3.2} & 19 & 2.271 & 0.0409 & 19/19 \\
\dvlModel{mistral}{Mistral} & 20 & 2.318 & 0.0417 & 20/20 \\
\dvlModel{ai2}{OLMo-2} & 20 & 2.450 & 0.0432 & 20/20 \\
\dvlModel{qwen}{Qwen1.5} & 20 & 2.339 & 0.0429 & 20/20 \\
\dvlModel{qwen}{Qwen2.5} & 16 & 2.442 & 0.0422 & 16/16 \\
\dvlModel{qwen}{Qwen3} & 20 & 2.379 & 0.0419 & 20/20 \\
\dvlModel{upstage}{SOLAR} & 20 & 2.368 & 0.0423 & 20/20 \\
\dvlModel{stability}{StableLM} & 20 & 2.397 & 0.0421 & 20/20 \\
\dvlModel{yi}{Yi-1.5} & 20 & 2.425 & 0.0428 & 20/20 \\
\bottomrule
\end{tabular*}
\end{minipage}\par

\subsection{Fixed Functional Inspection}
\label{app:frozen_functional_inspection}

We fix the selection rule before examining any token. We consider five models,
seven scopes (Q, K, V, O, FFN, Q+K+V+O, and Q+K+V+O+FFN), and directions with
high, median, and low cohesion, giving 105 single-direction cases. For each
case, we inspect both score tails.

Figures~\ref{fig:token_cases} and~\ref{fig:token_centroids} show two forms of structure in Mistral-7B. One Q
direction separates a group of field-related tokens from a group of
complete-related tokens. In the centroid examples, one group shows a clear
terrain-versus-grade pattern, while other groups contain many token fragments
and are harder to interpret. These examples show that some vocabulary tails
have clear and readable structure, while others are less easy to describe.
They therefore support local interpretability without assuming that every
parameter or cluster has a single clean meaning.

\begin{figure}[!t]
  \centering
  \includegraphics[width=0.98\linewidth]{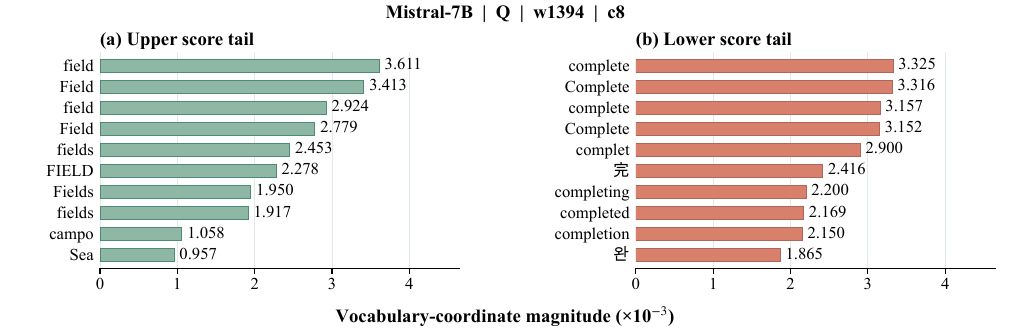}
  \caption{\textbf{A Single Direction Separates Two Readable Vocabulary-Tail Groups.}
  One fixed Mistral-7B query-projection direction: the upper score tail is dominated by variants of the word field and the lower score tail by variants of the word complete, with coordinate magnitudes and within-tail ordering retained. Panels (a) and (b) show the upper and lower tails, respectively; values are in units of $10^{-3}$. Same-model centroids are shown in Figure~\ref{fig:token_centroids}.}
  \label{fig:token_cases}
\end{figure}

\section{Intervention Details and Cases}
\label{app:interventions}

\subsection{Aggregate Intervention Setup and Denominators}

The aggregate results in Section~\ref{sec:intervention} cover 21 models, ten
concepts, six scopes, three operations, and eight intervention fractions.
Each model contributes
$10\times6\times3\times8=1{,}440$ concept-level cells, giving 30,240 cells in
total and 3,024 aggregate rows for each measure. The full-Jacobian/vector--Jacobian-product directions,
eligible layers, intervention strengths, generation-time edits, and model set
are kept fixed throughout the experiments.

\subsubsection*{Intervention Operators, Rankings, and Strengths.}

For a selected parameter column $w_i$, let $u_c$ be the unit-normalized
full-Jacobian/vector--Jacobian-product direction for concept $c$, let $u_t$ denote the paired target
direction, and let $\alpha$ be the fixed strength for the corresponding model,
operation, and scope. Ablation applies
$w_i\leftarrow w_i-\alpha u_c(u_c^\top w_i)$, while steering applies
$w_i\leftarrow w_i+\alpha\lVert w_i\rVert_2u_t$.
Swapping uses separate source and target rankings: source columns receive the
ablation update for $u_c$, while target columns receive the steering update
for $u_t$.

The edits remain active during next-token scoring, full-answer scoring, and
the complete 16-token continuation, after which all edited columns are
restored exactly. For each concept, scope, and operation, columns follow one
fixed full-Jacobian ranking. The eight intervention fractions are nested prefixes of
this ranking, with no filtering based on decoded outputs.

At the reference fraction of $0.2\%$, $\alpha$ is selected separately for
each model, operation, and scope from
${0.25,0.5,1,2,4,8}$. We choose the value with the highest mean of the
first-token and full-answer probability-direction indicators across the ten
concepts, breaking ties in favor of the smaller strength. This value is then
used for all reported fractions.

The conditional decoded success rate is evaluated only for concept cells
whose clean continuation contains the original answer. Among the 3,024
aggregate rows, 2,880 have a nonzero denominator and 144 are undefined.

\subsection{Complete Conditional Decoded Success-Rate Grid}
\label{app:decoded_all_models}

Figure~\ref{fig:intervention_all_models} extends Figure~\ref{fig:intervention_budget} to all 21 models with the same scopes, operations, and fractions.

\begin{figure}[!t]
  \centering
  {\fontsize{8}{9}\selectfont\input{paper_figures/FigD_decoded_all_21_labels}}
  \caption{\textbf{Conditional Decoded Success Rates for All 21 Models.}
  Layout as in Figure~\ref{fig:intervention_budget}; the bottom row is the matched random baseline averaged over defined model rows; grey cells have a zero denominator. Model codes extend those of Figure~\ref{fig:intervention_budget}: F3, G3 and G3.3, O2, S, and Y1.5 denote Falcon3, Granite-3.0 and Granite-3.3, OLMo-2, SOLAR, and Yi-1.5.}
  \label{fig:intervention_all_models}
\end{figure}

\subsection{Future-Token Probability Direction}
\label{app:probability_direction}

Figure~\ref{fig:probability_direction_appendix} reports the probability-based
measure separately from the conditional decoded success rate. A concept cell counts
as successful when at least one future answer token moves in the required
direction at one or more checked positions. Since all ten concepts contribute
to every aggregate row, all 3,024 rows are defined.

At Joint scope, the highest pooled rates are $209/210$ for ablation at
$1\%$, $3\%$, and $5\%$; $210/210$ for steering from $0.05\%$ through
$0.2\%$; and $198/210$ for swapping at $0.05\%$ and $0.2\%$.

This measure checks several future positions and answer tokens without
correcting for the number of checks. The high rates show that this broad
event is frequent under the fixed protocol.

\begin{figure}[!t]
  \centering
  {\fontsize{8}{9}\selectfont\input{paper_figures/FigD_logprob_all_21_labels}}
  \caption{\textbf{Future-Token Probability-Direction Events for the 21 Models.}
  Layout, scale, and model codes as in Figure~\ref{fig:intervention_all_models}. A concept counts when at least one future answer token moves in the specified direction at any checked position; all 3,024 rows are defined. This event is not a decoded-success measure and carries no multiplicity correction.}
  \label{fig:probability_direction_appendix}
\end{figure}

Figure~\ref{fig:probability_direction_appendix} shows high rates across many
models and module scopes, especially for steering at small fractions, while
differences across operations and scopes remain visible. The near-saturation
of some settings is consistent with this event being relatively easy to
satisfy.

\subsection{Matched Random-Column and Same-Norm Control}
\label{app:random-control}

Each of the 3,024 primary model--scope--operation--fraction rows has a matched
random-control row over the same concepts. The control keeps the checkpoint,
prompt, module scope, eligible matrices and layers, operation, intervention
fraction, strength, clean-generation denominator, future-token positions, and
evaluation rules unchanged.

Within each target matrix, the selected columns are replaced by the same
number of distinct non-primary columns. These columns are sampled from within
$\pm5\%$ of the original weight-column norm rank, using the nearest available
rank when necessary. For each target, we then sample an independent random
direction and rescale the update so that both the columnwise and total
Frobenius norms match the paired primary intervention. A cell is discarded if
the relative norm mismatch is larger than $2.5\%$.

One matched random draw is generated for every primary cell using a fixed
seed. Figures~\ref{fig:intervention_budget} and \ref{fig:intervention_all_models} show the model-averaged control as
their bottom row; Figure~\ref{fig:random_intervention_control} shows the full
control matrix. Averages include only model rows with a defined paired conditional
decoded denominator; undefined cases are shown in grey.

Across all 144 matched scope--operation--fraction settings, the mean
conditional decoded success rate is $30.82\%$ for the primary intervention
and $3.13\%$ for the control. At Joint scope, the corresponding averages are
$49.49\%$ and $3.27\%$. The overall difference of $27.69$ percentage points
is based on this single matched random draw and does not measure uncertainty
across repeated random controls.

\begin{figure}[!t]
  \centering
  {\fontsize{8}{9}\selectfont\input{paper_figures/FigD_random_all_21_labels}}
  \caption{\textbf{Matched Random-Column and Same-Norm Control for the 21 Models.}
  Each control cell keeps the primary cell's model, concept, scope, operation, fraction, strength, eligible layers, denominator, and evaluation rule, but uses norm-matched non-primary columns and isotropic directions rescaled to the primary update norms. Model codes as in Figure~\ref{fig:intervention_all_models}; grey cells have a zero denominator. The conditional decoded success rate averages $3.13\%$ against $30.82\%$ for the primary selection; this single draw is a reference control, not a null distribution.}
  \label{fig:random_intervention_control}
\end{figure}

Figure~\ref{fig:random_intervention_control} shows that the matched control has
low rates across all six scopes and does not reproduce the broad high-rate
regions of the primary intervention in Figure~\ref{fig:intervention_all_models}.
A small number of nonzero control cells remain, so the comparison does not
imply that every random intervention has no effect.

\subsection{Qualitative Intervention Cases}
\label{app:qualitative_cases}

Tables~\ref{tab:e_case_granite}--\ref{tab:e_case_stablelm} present 36 selected
cases across six models, with one case for each of Q, K, V, O, FFN, and Joint.
We select models with complete module coverage and cases with readable
continuations that satisfy the recorded decoded-success criterion.

The tables combine the archived fixed-$2\%$ replay run with successful cells
from the factual intervention grid in Figure~\ref{fig:intervention_all_models}.
Each row retains its original operation, perturbed fraction, and strength
$\alpha$, followed by the prompt and clean and edited continuation excerpts.
For ablation, the edited text omits the source answer; steering and swapping
cases show the paired target. Ellipses indicate omitted text.
Figure~\ref{fig:intervention_cases_main} highlights four factual cases in the main text.

\begin{table*}[!t]
  \centering
  \begingroup
  \fontsize{9}{10.4}\selectfont
  \setlength{\tabcolsep}{3pt}
  \renewcommand{\arraystretch}{1.09}
  \caption{\dvlModel{ibm}{Granite-3.3-8B}: decoded intervention cases across all six scopes.}
  \label{tab:e_case_granite}
  \label{tab:intervention_cases_all}
  \begin{tabularx}{\textwidth}{@{}>{\raggedright\arraybackslash}p{53pt}>{\raggedright\arraybackslash}p{105pt}>{\columncolor{cyan!3}\raggedright\arraybackslash}X>{\columncolor{orange!4}\raggedright\arraybackslash}X@{}}
  \toprule
  \textbf{Scope / Edit} & \textbf{Prompt Excerpt} & \textcolor{cyan!55!black}{\textbf{Clean}} & \textcolor{orange!80!black}{\textbf{Edited}} \\
  \midrule
  \textbf{Q} $\cdot$ Ablate\newline 2\%, $\alpha{=}0.5$\newline {\fontsize{8}{9}\selectfont Replay} & Discounted cash flow is used for valuation in this field. \ldots{} & finance \ldots{} & programming \ldots{} \\
  \addlinespace[3pt]
  \cmidrule(lr){1-4}
  \textbf{K} $\cdot$ Swap\newline 2\%, $\alpha{=}2$\newline {\fontsize{8}{9}\selectfont Replay} & \ldots{} An isotope has a characteristic number of protons but a different number of neutrons here. \ldots{} & chemistry \ldots{} & Physics \ldots{} \\
  \addlinespace[3pt]
  \cmidrule(lr){1-4}
  \textbf{V} $\cdot$ Steer\newline 2\%, $\alpha{=}0.5$\newline {\fontsize{8}{9}\selectfont Replay} & \ldots{} The Dolomites rise in the northeastern part of this country. \ldots{} & Italy \ldots{} & France \ldots{} \\
  \addlinespace[3pt]
  \cmidrule(lr){1-4}
  \textbf{O} $\cdot$ Steer\newline 0.2\%, $\alpha{=}8$\newline {\fontsize{8}{9}\selectfont Factual} & Norbert Wiener's expertise is & in the field of mathematics and engineering, particularly in the areas of control systems & evident in the harmony and rhythm of the music, and the \\
  \addlinespace[3pt]
  \cmidrule(lr){1-4}
  \textbf{FFN} $\cdot$ Ablate\newline 2\%, $\alpha{=}0.25$\newline {\fontsize{8}{9}\selectfont Replay} & Choirs coordinate pitch, phrasing, and vocal blend in this discipline. \ldots{} & music \ldots{} & mathematics \ldots{} \\
  \addlinespace[3pt]
  \cmidrule(lr){1-4}
  \textbf{Joint} $\cdot$ Swap\newline 2\%, $\alpha{=}0.5$\newline {\fontsize{8}{9}\selectfont Replay} & \ldots{} Reaction heat and enthalpy are topics in this discipline. \ldots{} & chemistry \ldots{} & Physics \ldots{} \\
  \bottomrule
  \end{tabularx}
  \endgroup
\end{table*}
\begin{table*}[!t]
  \centering
  \begingroup
  \fontsize{9}{10.4}\selectfont
  \setlength{\tabcolsep}{3pt}
  \renewcommand{\arraystretch}{1.09}
  \caption{\dvlModel{meta}{Llama-3.2-3B}: decoded intervention cases across all six scopes.}
  \label{tab:e_case_llama32}
  \begin{tabularx}{\textwidth}{@{}>{\raggedright\arraybackslash}p{53pt}>{\raggedright\arraybackslash}p{105pt}>{\columncolor{cyan!3}\raggedright\arraybackslash}X>{\columncolor{orange!4}\raggedright\arraybackslash}X@{}}
  \toprule
  \textbf{Scope / Edit} & \textbf{Prompt Excerpt} & \textcolor{cyan!55!black}{\textbf{Clean}} & \textcolor{orange!80!black}{\textbf{Edited}} \\
  \midrule
  \textbf{Q} $\cdot$ Ablate\newline 2\%, $\alpha{=}0.5$\newline {\fontsize{8}{9}\selectfont Replay} & Molecular spectroscopy identifies substances in this science. \ldots{} & Chemistry. \ldots{} & Physics. \ldots{} \\
  \addlinespace[3pt]
  \cmidrule(lr){1-4}
  \textbf{K} $\cdot$ Ablate\newline 0.2\%, $\alpha{=}2$\newline {\fontsize{8}{9}\selectfont Factual} & Shawn Fanning's profession is a & computer programmer and entrepreneur. He is best known for co-founding Napster, & professional wrestler, actor, and stuntman. He is best known for his role \\
  \addlinespace[3pt]
  \cmidrule(lr){1-4}
  \textbf{V} $\cdot$ Swap\newline 0.05\%, $\alpha{=}8$\newline {\fontsize{8}{9}\selectfont Factual} & William Herschel's expertise is & in the field of astronomy, particularly in the study of the universe's structure and & in the field of biology, genetics, and evolutionary biology. He is a renowned \\
  \addlinespace[3pt]
  \cmidrule(lr){1-4}
  \textbf{O} $\cdot$ Swap\newline 0.05\%, $\alpha{=}4$\newline {\fontsize{8}{9}\selectfont Factual} & William Herschel's expertise is & in the field of astronomy, particularly in the study of the universe's structure and & in genetics and biology, and he is a renowned biologist. He is a cell \\
  \addlinespace[3pt]
  \cmidrule(lr){1-4}
  \textbf{FFN} $\cdot$ Swap\newline 2\%, $\alpha{=}0.5$\newline {\fontsize{8}{9}\selectfont Replay} & Fugue, symphony, concerto, and sonata are forms in this discipline. \ldots{} & music. \ldots{} & mathematics. \ldots{} \\
  \addlinespace[3pt]
  \cmidrule(lr){1-4}
  \textbf{Joint} $\cdot$ Ablate\newline 0.05\%, $\alpha{=}8$\newline {\fontsize{8}{9}\selectfont Factual} & William Herschel's expertise is & in the field of astronomy, particularly in the study of the universe's structure and & in the field of chemistry, particularly in the area of organic chemistry. His research \\
  \bottomrule
  \end{tabularx}
  \endgroup
\end{table*}
\begin{table*}[!t]
  \centering
  \begingroup
  \fontsize{9}{10.4}\selectfont
  \setlength{\tabcolsep}{3pt}
  \renewcommand{\arraystretch}{1.09}
  \caption{\dvlModel{meta}{Llama-3.1-8B}: decoded intervention cases across all six scopes.}
  \label{tab:e_case_llama31}
  \begin{tabularx}{\textwidth}{@{}>{\raggedright\arraybackslash}p{53pt}>{\raggedright\arraybackslash}p{105pt}>{\columncolor{cyan!3}\raggedright\arraybackslash}X>{\columncolor{orange!4}\raggedright\arraybackslash}X@{}}
  \toprule
  \textbf{Scope / Edit} & \textbf{Prompt Excerpt} & \textcolor{cyan!55!black}{\textbf{Clean}} & \textcolor{orange!80!black}{\textbf{Edited}} \\
  \midrule
  \textbf{Q} $\cdot$ Ablate\newline 0.5\%, $\alpha{=}1$\newline {\fontsize{8}{9}\selectfont Factual} & The domain of work of Henri Becquerel is & a very important part of the history of physics. He was a French physicist who & a good example of the way in which the scientific community has been able to recognize \\
  \addlinespace[3pt]
  \cmidrule(lr){1-4}
  \textbf{K} $\cdot$ Ablate\newline 0.05\%, $\alpha{=}4$\newline {\fontsize{8}{9}\selectfont Factual} & William Herschel's expertise is & in the field of astronomy, particularly in the study of the solar system and the & in the field of mathematics, specifically in the area of algebraic geometry. His \\
  \addlinespace[3pt]
  \cmidrule(lr){1-4}
  \textbf{V} $\cdot$ Swap\newline 0.5\%, $\alpha{=}2$\newline {\fontsize{8}{9}\selectfont Factual} & Norbert Wiener's expertise is & in the field of mathematics, particularly in the areas of harmonic analysis, stochastic processes & in the field of music, and he is a renowned composer and pianist. \\
  \addlinespace[3pt]
  \cmidrule(lr){1-4}
  \textbf{O} $\cdot$ Swap\newline 2\%, $\alpha{=}1$\newline {\fontsize{8}{9}\selectfont Replay} & \ldots{} Arithmetic operations on fractions are elementary work in this subject. \ldots{} & mathematics. (Fractions are a fundamental concept in mathematics.) (Source: Stanford & music \ldots{} \\
  \addlinespace[3pt]
  \cmidrule(lr){1-4}
  \textbf{FFN} $\cdot$ Steer\newline 0.1\%, $\alpha{=}0.5$\newline {\fontsize{8}{9}\selectfont Factual} & William Herschel's expertise is & in the field of astronomy, particularly in the study of the solar system and the & in biology and genetics, with a focus on the biology of aging and the genetics \\
  \addlinespace[3pt]
  \cmidrule(lr){1-4}
  \textbf{Joint} $\cdot$ Swap\newline 2\%, $\alpha{=}0.25$\newline {\fontsize{8}{9}\selectfont Replay} & \ldots{} The euro-using republic between Spain and Germany is this country. \ldots{} & France. (The country between Spain and Germany is France, and the country that & Italy. (Italy is the country between Spain and Germany, and it uses the \\
  \bottomrule
  \end{tabularx}
  \endgroup
\end{table*}
\begin{table*}[!t]
  \centering
  \begingroup
  \fontsize{9}{10.4}\selectfont
  \setlength{\tabcolsep}{3pt}
  \renewcommand{\arraystretch}{1.09}
  \caption{\dvlModel{mistral}{Mistral-7B}: decoded intervention cases across all six scopes.}
  \label{tab:e_case_mistral}
  \begin{tabularx}{\textwidth}{@{}>{\raggedright\arraybackslash}p{53pt}>{\raggedright\arraybackslash}p{105pt}>{\columncolor{cyan!3}\raggedright\arraybackslash}X>{\columncolor{orange!4}\raggedright\arraybackslash}X@{}}
  \toprule
  \textbf{Scope / Edit} & \textbf{Prompt Excerpt} & \textcolor{cyan!55!black}{\textbf{Clean}} & \textcolor{orange!80!black}{\textbf{Edited}} \\
  \midrule
  \textbf{Q} $\cdot$ Swap\newline 2\%, $\alpha{=}4$\newline {\fontsize{8}{9}\selectfont Replay} & \ldots{} A spectrometer can identify a substance from the wavelengths it absorbs in this science. \ldots{} & chemistry \ldots{} & Physics \ldots{} \\
  \addlinespace[3pt]
  \cmidrule(lr){1-4}
  \textbf{K} $\cdot$ Ablate\newline 0.5\%, $\alpha{=}2$\newline {\fontsize{8}{9}\selectfont Factual} & Norbert Wiener's expertise is & in the field of mathematics, but he is also known for his work in & in the field of international law, international relations, and international organizations. He is \\
  \addlinespace[3pt]
  \cmidrule(lr){1-4}
  \textbf{V} $\cdot$ Steer\newline 0.05\%, $\alpha{=}4$\newline {\fontsize{8}{9}\selectfont Factual} & Norbert Wiener's expertise is & in the field of mathematics, but he is also known for his work in & in the field of music, and he is a composer of some distinction. He \\
  \addlinespace[3pt]
  \cmidrule(lr){1-4}
  \textbf{O} $\cdot$ Swap\newline 2\%, $\alpha{=}2$\newline {\fontsize{8}{9}\selectfont Replay} & \ldots{} Ecological succession describes how communities change after disturbance in this field. \ldots{} & biology & astronomy \\
  \addlinespace[3pt]
  \cmidrule(lr){1-4}
  \textbf{FFN} $\cdot$ Ablate\newline 2\%, $\alpha{=}0.5$\newline {\fontsize{8}{9}\selectfont Replay} & \ldots{} This euro-using country lies between Austria and the Mediterranean Sea. \ldots{} & Italy \ldots{} & France \ldots{} \\
  \addlinespace[3pt]
  \cmidrule(lr){1-4}
  \textbf{Joint} $\cdot$ Swap\newline 2\%, $\alpha{=}0.5$\newline {\fontsize{8}{9}\selectfont Replay} & \ldots{} Testimony is given under oath and may be challenged by cross-examination here. \ldots{} & Law & Medicine \ldots{} \\
  \bottomrule
  \end{tabularx}
  \endgroup
\end{table*}
\begin{table*}[!t]
  \centering
  \begingroup
  \fontsize{9}{10.4}\selectfont
  \setlength{\tabcolsep}{3pt}
  \renewcommand{\arraystretch}{1.09}
  \caption{\dvlModel{qwen}{Qwen2.5-3B}: decoded intervention cases across all six scopes.}
  \label{tab:e_case_qwen3}
  \begin{tabularx}{\textwidth}{@{}>{\raggedright\arraybackslash}p{53pt}>{\raggedright\arraybackslash}p{105pt}>{\columncolor{cyan!3}\raggedright\arraybackslash}X>{\columncolor{orange!4}\raggedright\arraybackslash}X@{}}
  \toprule
  \textbf{Scope / Edit} & \textbf{Prompt Excerpt} & \textcolor{cyan!55!black}{\textbf{Clean}} & \textcolor{orange!80!black}{\textbf{Edited}} \\
  \midrule
  \textbf{Q} $\cdot$ Ablate\newline 0.05\%, $\alpha{=}8$\newline {\fontsize{8}{9}\selectfont Factual} & Grand Mosque of Paris, which is located in & the 10th arrondissement of Paris, France, is a mosque & the 12th arrondissement, is a mosque for Muslims. It \\
  \addlinespace[3pt]
  \cmidrule(lr){1-4}
  \textbf{K} $\cdot$ Ablate\newline 0.05\%, $\alpha{=}0.25$\newline {\fontsize{8}{9}\selectfont Factual} & Grand Mosque of Paris, which is located in & the 10th arrondissement of Paris, France, is a mosque & the 10th arrondissement, is a mosque that was built in \\
  \addlinespace[3pt]
  \cmidrule(lr){1-4}
  \textbf{V} $\cdot$ Steer\newline 0.2\%, $\alpha{=}2$\newline {\fontsize{8}{9}\selectfont Factual} & Limoncello, that was created in & the south of Italy, is a clear, yellow liqueur made from the zest & the south of France, is a clear, yellow liqueur made from the zest \\
  \addlinespace[3pt]
  \cmidrule(lr){1-4}
  \textbf{O} $\cdot$ Steer\newline 0.2\%, $\alpha{=}0.5$\newline {\fontsize{8}{9}\selectfont Factual} & Limoncello, that was created in & the south of Italy, is a clear, yellow liqueur made from the zest & the south of France, is a clear and very light yellow liqueur made from \\
  \addlinespace[3pt]
  \cmidrule(lr){1-4}
  \textbf{FFN} $\cdot$ Swap\newline 0.05\%, $\alpha{=}4$\newline {\fontsize{8}{9}\selectfont Factual} & William Herschel's expertise is & in the field of astronomy. He discovered Uranus and made significant contributions to the & in the field of biology, specifically in the area of bioinformatics. He has \\
  \addlinespace[3pt]
  \cmidrule(lr){1-4}
  \textbf{Joint} $\cdot$ Steer\newline 0.05\%, $\alpha{=}0.25$\newline {\fontsize{8}{9}\selectfont Factual} & Limoncello, that was created in & the south of Italy, is a clear, yellow liqueur made from the zest & the south of France, is a clear yellow liqueur made from the zest of \\
  \bottomrule
  \end{tabularx}
  \endgroup
\end{table*}
\begin{table*}[!t]
  \centering
  \begingroup
  \fontsize{9}{10.4}\selectfont
  \setlength{\tabcolsep}{3pt}
  \renewcommand{\arraystretch}{1.09}
  \caption{\dvlModel{stability}{StableLM-2-1.6B}: decoded intervention cases across all six scopes.}
  \label{tab:e_case_stablelm}
  \begin{tabularx}{\textwidth}{@{}>{\raggedright\arraybackslash}p{53pt}>{\raggedright\arraybackslash}p{105pt}>{\columncolor{cyan!3}\raggedright\arraybackslash}X>{\columncolor{orange!4}\raggedright\arraybackslash}X@{}}
  \toprule
  \textbf{Scope / Edit} & \textbf{Prompt Excerpt} & \textcolor{cyan!55!black}{\textbf{Clean}} & \textcolor{orange!80!black}{\textbf{Edited}} \\
  \midrule
  \textbf{Q} $\cdot$ Steer\newline 2\%, $\alpha{=}0.5$\newline {\fontsize{8}{9}\selectfont Replay} & \ldots{} The canals and gondolas of Venice attract visitors to this country. \ldots{} & Italy \ldots{} & France \ldots{} \\
  \addlinespace[3pt]
  \cmidrule(lr){1-4}
  \textbf{K} $\cdot$ Ablate\newline 1\%, $\alpha{=}0.25$\newline {\fontsize{8}{9}\selectfont Factual} & Grand Mosque of Paris, which is located in & the heart of Paris, France, is one of the most famous and largest mosques & the heart of Paris, is one of the most famous and iconic mosques in the \\
  \addlinespace[3pt]
  \cmidrule(lr){1-4}
  \textbf{V} $\cdot$ Swap\newline 2\%, $\alpha{=}8$\newline {\fontsize{8}{9}\selectfont Replay} & \ldots{} Diversification reduces concentration risk in this discipline. \ldots{} & Diversification reduces concentration risk in finance. \ldots{} & Diversification reduces concentration risk in the field of programming. \ldots{} \\
  \addlinespace[3pt]
  \cmidrule(lr){1-4}
  \textbf{O} $\cdot$ Ablate\newline 2\%, $\alpha{=}4$\newline {\fontsize{8}{9}\selectfont Replay} & A researcher examining DNA inheritance works in this branch of science. \ldots{} & The researcher examines DNA inheritance in the field of biology. \ldots{} & A researcher examining DNA inheritance works in the field of astronomy. \ldots{} \\
  \addlinespace[3pt]
  \cmidrule(lr){1-4}
  \textbf{FFN} $\cdot$ Swap\newline 2\%, $\alpha{=}4$\newline {\fontsize{8}{9}\selectfont Replay} & \ldots{} A white coat, stethoscope, and examination room are associated with this profession. \ldots{} & Medical Doctor (MD) \ldots{} & Law Enforcement Officer \ldots{} \\
  \addlinespace[3pt]
  \cmidrule(lr){1-4}
  \textbf{Joint} $\cdot$ Ablate\newline 2\%, $\alpha{=}2$\newline {\fontsize{8}{9}\selectfont Replay} & A hospital resident training under senior physicians works in this profession. \ldots{} & A hospital resident training under senior physicians works in the profession of medicine. \ldots{} & A resident in a specific field of study, such as law or a related field \\
  \bottomrule
  \end{tabularx}
  \endgroup
\end{table*}

\clearpage

\section{Additional \dvl{} Computation Details}
\label{app:dvl_computation}

This section records the contractions, the reusable operators, and the input-conditioned scores used by the reported experiments.

\subsubsection*{Concept-Space Contractions.}

The parameter signal requires only the Jacobian contraction
$\bar J_{\ell,m}^{\top}s_c$, rather than the full Jacobian
$\bar J_{\ell,m}$.  For each prompt $x$, this contraction can be
computed directly as
\begin{equation}
    J_{\ell,m\rightarrow L}(x)^\top s_c
    =
    \nabla_{h_{\ell,m}(x)}
    \left[
        s_c^\top z(x)
    \right],
    \label{eq:concept_vjp}
\end{equation}
so that
\begin{equation}
    d_{\ell,m,c}
    =
    \frac{1}{|\mathcal D_J|}
    \sum_{x\in\mathcal D_J}
    \nabla_{h_{\ell,m}(x)}
    \left[
        s_c^\top z(x)
    \right].
    \label{eq:concept_vjp_average}
\end{equation}
A vector--Jacobian product therefore produces the required
$d_{\ell,m,c}$ without forming $J_{\ell,m\rightarrow L}(x)$, and
several concepts can be evaluated together.

When concepts from the same output space are queried repeatedly, we
further reuse these computations across concepts.  Let
$\mathcal S=\operatorname{span}\{s_1,\ldots,s_C\}$ and let
$Q\in\mathbb{R}^{r\times d_{\mathrm{model}}}$ contain an orthonormal
basis of $\mathcal S$ in its rows.  We precompute
$B_{\ell,m}=Q\bar J_{\ell,m}$.  Since
$Q^\top Qs_c=s_c$ for every $s_c\in\mathcal S$,
\begin{equation}
    d_{\ell,m,c}
    =
    \bar J_{\ell,m}^{\top}s_c
    =
    B_{\ell,m}^{\top}Qs_c.
    \label{eq:reusable_signal}
\end{equation}
Hence $B_{\ell,m}$ is computed once for the concept space, after which
each concept signal is recovered by one matrix multiplication,
without a further backward pass through the model.

\subsubsection*{Input-Conditioned Intervention Scores.}

The parameter signal captures the association between a concept and
learned parameters over the Jacobian-estimation prompts.  For a particular
intervention, input-dependent activations identify the parameters
participating in the relevant behavior.

For concept $c$, define
$\ell_c(x)=\log\sum_{v\in\mathcal T_c}p(v\mid x)$ and
$\widehat d_{\ell,m,c}
=d_{\ell,m,c}/(\|d_{\ell,m,c}\|_2+\varepsilon)$.
At position $t$, we measure the sensitivity of the concept probability
to this signal as
\begin{equation}
    q_{\ell,m,c,t}(x)
    =
    \left(
        \nabla_{h_{\ell,m,t}}
        \ell_c(x)
    \right)^\top
    \widehat d_{\ell,m,c}.
    \label{eq:concept_sensitivity}
\end{equation}
Combining this sensitivity with the corresponding input-dependent
activation gives
\begin{equation}
    \alpha_{\ell,m,i,c}(x)
    =
    \frac{1}{|\mathcal P(x)|}
    \sum_{t\in\mathcal P(x)}
    q_{\ell,m,c,t}(x)
    f_{\ell,m,i,t}(x).
    \label{eq:parameter_sensitivity}
\end{equation}
For removing an existing concept, we use
$M^{\mathrm{remove}}_{\ell,m,i,c}
=\alpha_{\ell,m,i,c}\beta_{\ell,m,i,c}$;
for adding a target concept, we use
$M^{\mathrm{add}}_{\ell,m,i,c}
=\alpha_{\ell,m,i,c}\|w_{\ell,m,i}\|_2$.
We aggregate these quantities over calibration prompts and rank the
eligible parameter columns within the specified module scope.

\subsubsection*{Shared Module-Level Readout Operator.}

All parameters in the same weight matrix share the same mapping from
the module output to the final vocabulary.  We therefore define the
module-level readout operator
$
    A_a
    =
    \WU\bar J_a.
$
For an individual parameter $w_{a,i}$, its \dvl{} readout is simply
$\bar r_{a,i}=A_aw_{a,i}$.  Since
$W_a=[w_{a,1},\ldots,w_{a,n_a}]$, the readouts of all parameters can
be computed jointly as
\begin{equation}
    R_a
    =
    A_aW_a
    =
    \WU\bar J_aW_a.
    \label{eq:dvl_matrix_readout}
\end{equation}
The mapping $A_a$ is therefore computed once for each module and
shared across all of its parameter columns, after which the full
\dvl{} readout matrix follows from one matrix multiplication.
The columns of $R_a$ are then normalized and used directly for
clustering.

\allowdisplaybreaks
\newcommand{\LSE}{\operatorname{LSE}}
\section{Theory of the \dvlkm{} Score}
\label{app:score_theory}

The \dvlkm{} score is the normalized parameter energy removed by
replacing each column with its role centroid. We first establish this
identity, then relate the removed energy to the first-order and finite
cross-entropy costs of the replacement.

\subsection{Setting}
\label{sec:score_setting}

\subsubsection*{Role Models and Score.}

Let $\mathcal A$ be a finite, nonempty collection of analysed matrices
$W_a=[w_{a,1},\ldots,w_{a,n_a}]\in\mathbb R^{d_a\times n_a}$,
with $n_a\ge1$ and $m_a:=\norm{W_a}_F^2>0$. Fix an integer $k\ge1$
and a labelling $c_{a,i}\in\{1,\ldots,k\}$. All cluster sums below
range over nonempty clusters. With
$C_{a,c}:=\{i:c_{a,i}=c\}$ and $n_{a,c}:=|C_{a,c}|$, define
\begin{equation}
 \bar w_{a,c}=\frac1{n_{a,c}}\sum_{i\in C_{a,c}}w_{a,i},
 \qquad
 \bar W_a[:,i]=\bar w_{a,c_{a,i}},
 \qquad \Delta_a=W_a-\bar W_a.
 \label{eq:role_centroid}
\end{equation}
Thus each $\bar W_a$ has at most $k$ distinct columns. Write
$W=(W_a)_a$, $\bar W=(\bar W_a)_a$, $\Delta=W-\bar W$, and use
$\langle U,V\rangle=\sum_a\langle U_a,V_a\rangle_F$ and
$\norm U^2=\sum_a\norm{U_a}_F^2$ for tuples of matrices. The scores
in Equations~\ref{eq:matrix_dvl_kmeans}--\ref{eq:model_dvl_kmeans} are
\begin{equation}
 S_a=\frac{\norm{\Delta_a}_F^2}{m_a},
 \qquad \bar S=\frac1{|\mathcal A|}\sum_a S_a,
 \qquad \mathcal E_{\mathrm{role}}:=\norm\Delta^2=\sum_a S_am_a.
 \label{eq:score_ratio}
\end{equation}

\subsubsection*{Prediction Loss and Regularity.}

Let $z_W(x)\in\mathbb R^{|\mathcal V|}$ be the logits with all
unanalysed parameters held fixed, and let
$p_W(\cdot\mid x)=\operatorname{softmax}(z_W(x))$. Define
\begin{equation}
 \ell_y(z)=\LSE(z)-z_y,
 \qquad \LSE(z)=\log\sum_v e^{z_v},
 \qquad L(W)=\mathbb E_{(x,y)}\ell_y(z_W(x)).
 \label{eq:population_loss}
\end{equation}
Assume finite logits almost surely, finite $L(W)$ and $L(\bar W)$,
and differentiability at $W$, with differentiation commuting with the
expectation defining $L$. All expectations below are finite.
The labelling and the centroids are fixed at the trained $W$ when taking
directional derivatives. Set
\begin{equation}
 J_W(x)=D_Wz_W(x),\qquad
 \delta z(x)=z_{\bar W}(x)-z_W(x)=d(x)+r(x),\qquad
 d(x)=-J_W(x)\Delta.
 \label{eq:first_order_logit_change}
\end{equation}
Here $r$ is the exact remainder of the finite displacement. For
$\omega\ge0$, define
\begin{equation}
 \rho=\nabla_WL(W)+\omega W,\qquad
 \varepsilon=\norm\rho,\qquad
 \varepsilon_a=\norm{\rho_a}_F.
 \label{eq:gradient_residual_rearranged}
\end{equation}
Thus $\rho=0$ means stationarity of
$L(W)+\frac\omega2\norm W^2$ on the analysed parameters.
Write $DL(W)[V]=\langle\nabla_WL(W),V\rangle$ and define
\begin{align}
 \kappa_1&=\mathbb E\langle p_W-e_y,d\rangle,
 &R&=\mathbb E\langle p_W-e_y,r\rangle,
 \label{eq:kappa_definition}\\
 \bar D&=\mathbb E_x\mathrm{KL}(p_W(\cdot\mid x)\|p_{\bar W}(\cdot\mid x)),
 &\rho_R&=2\mathbb E_x\norm{r(x)}_\infty,
 \label{eq:D_definition}\\
 \Lambda_1&=\mathbb E_x\norm{J_W(x)}_{F\to\infty}. &&
 \label{eq:jacobian_sensitivity}
\end{align}
The operator norm uses the tuple norm above.

\subsection{Energy Removed by Role Collapse}

\begin{lemma}[Orthogonal decomposition induced by a labelling]
\label{lem:energy_split}
For every fixed labelling,
\begin{equation}
 \langle\bar W_a,\Delta_a\rangle_F=0,\qquad
 m_a=\norm{\bar W_a}_F^2+\norm{\Delta_a}_F^2,\qquad
 S_a=1-\frac{\norm{\bar W_a}_F^2}{m_a}\in[0,1].
 \label{eq:matrix_pythagorean}
\end{equation}
These identities imply
\begin{align}
 \langle W,\Delta\rangle&=\norm\Delta^2=\sum_a S_am_a\le\norm W^2,
 \label{eq:score_bridge}\\
 |\mathcal A|\bar S\min_a m_a&\le\norm\Delta^2
 \le |\mathcal A|\bar S\max_a m_a.
 \label{eq:average_score_bridge}
\end{align}
\end{lemma}
\begin{proof}
The residuals sum to zero in each cluster. Hence
\begin{align*}
 \langle\bar W_a,\Delta_a\rangle_F
 &=\sum_c\left\langle\bar w_{a,c},
       \sum_{i\in C_{a,c}}(w_{a,i}-\bar w_{a,c})\right\rangle=0,\\
 m_a&=\norm{\bar W_a+\Delta_a}_F^2
     =\norm{\bar W_a}_F^2+\norm{\Delta_a}_F^2,\\
 S_a&=\frac{\norm{\Delta_a}_F^2}{m_a}
     =1-\frac{\norm{\bar W_a}_F^2}{m_a}\in[0,1],\\
 \langle W,\Delta\rangle
 &=\sum_a\langle\bar W_a+\Delta_a,\Delta_a\rangle_F
   =\sum_a\norm{\Delta_a}_F^2=\sum_a S_am_a\le\norm W^2,\\
 \left(\min_a m_a\right)\sum_aS_a
 &\le\sum_a S_am_a\le\left(\max_a m_a\right)\sum_aS_a.
 \qedhere
\end{align*}
\end{proof}

\subsection{Cross-Entropy under a Logit Change}

\begin{lemma}[Softmax cross-entropy identity and bounds]
\label{lem:bregman}
For finite logits $z,z'$ and $p=\operatorname{softmax}(z)$,
$q=\operatorname{softmax}(z')$,
\begin{align}
 \ell_y(z')-\ell_y(z)&=\langle p-e_y,z'-z\rangle+\mathrm{KL}(p\|q),
 \label{eq:bregman}\\
 0\le\mathrm{KL}(p\|q)&\le2\norm{z'-z}_\infty,
 \qquad |\ell_y(z')-\ell_y(z)|\le2\norm{z'-z}_\infty.
 \label{eq:ce_lipschitz}
\end{align}
\end{lemma}
\begin{proof}
Put $h=z'-z$. The softmax log-probabilities and nonnegativity of Kullback--Leibler (KL) divergence give
\begin{align*}
 \mathrm{KL}(p\|q)
 &=\sum_vp_v(z_v-z'_v)+\LSE(z')-\LSE(z),\\
 \ell_y(z')-\ell_y(z)
 &=\mathrm{KL}(p\|q)+\langle p,h\rangle-h_y
   =\mathrm{KL}(p\|q)+\langle p-e_y,h\rangle,\\
 0\le\mathrm{KL}(p\|q)
 &\le\mathrm{KL}(p\|q)+\mathrm{KL}(q\|p)
   =\langle p-q,z-z'\rangle
   \le2\norm h_\infty,\\
 \ell_y(z')-\ell_y(z)
 &=\int_0^1\langle\operatorname{softmax}(z+th)-e_y,h\rangle\,dt,\\
 |\ell_y(z')-\ell_y(z)|
 &\le\int_0^1\norm{\operatorname{softmax}(z+th)-e_y}_1
                  \norm h_\infty\,dt
   \le2\norm h_\infty.
 \qedhere
\end{align*}
\end{proof}

\subsection{Exact Loss Decomposition}

\begin{theorem}[Exact role-collapse decomposition]
\label{thm:collapse_decomposition}
Under the setting above, for every labelling and $\omega\ge0$,
\begin{align}
 L(\bar W)-L(W)&=\kappa_1+\bar D+R,
 &\kappa_1&=\omega\norm\Delta^2-\langle\rho,\Delta\rangle,
 \label{eq:exact_collapse_decomposition}\\
 |\kappa_1-\omega\norm\Delta^2|&\le\varepsilon\norm\Delta,
 &|R|&\le\rho_R.
 \label{eq:first_order_residual_bound}
\end{align}
At stationarity, $\kappa_1=\omega\sum_aS_am_a$.
\end{theorem}
\begin{proof}
Using the chain rule under the stated expectation condition and
Lemma~\ref{lem:bregman},
\begin{align*}
 \nabla_WL(W)
 &=\mathbb E\big[J_W(x)^\top(p_W(x)-e_y)\big],\\
 L(\bar W)-L(W)
 &=\mathbb E\langle p_W-e_y,\delta z\rangle+\bar D\\
 &=\mathbb E\langle p_W-e_y,d\rangle
   +\mathbb E\langle p_W-e_y,r\rangle+\bar D
   =\kappa_1+R+\bar D,\\
 \kappa_1
 &=-\mathbb E\langle p_W-e_y,J_W\Delta\rangle
   =-\langle\nabla_WL(W),\Delta\rangle\\
 &=-\langle\rho-\omega W,\Delta\rangle
   =\omega\norm\Delta^2-\langle\rho,\Delta\rangle,\\
 |\kappa_1-\omega\norm\Delta^2|
 &=|\langle\rho,\Delta\rangle|\le\varepsilon\norm\Delta,\\
 |R|&\le\mathbb E\big[\norm{p_W-e_y}_1\norm r_\infty\big]
      \le2\mathbb E_x\norm r_\infty=\rho_R,\\
 \rho=0&\quad\Longrightarrow\quad
 \kappa_1=\omega\norm\Delta^2=\omega\sum_aS_am_a.
 \qedhere
\end{align*}
\end{proof}

\begin{corollary}[Local collapse cost at stationarity]
\label{cor:no_useless_spread}
If $\omega>0$, $\rho=0$, and $\Delta\ne0$, the unit collapse direction
$u=-\Delta/\norm\Delta$ satisfies
\begin{equation}
 DL(W)[u]=\omega\norm\Delta
 =\omega\left(\sum_aS_am_a\right)^{1/2}>0.
 \label{eq:collapse_directional_stationary}
\end{equation}
For all $\Delta$, the unnormalized derivative is
$DL(W)[-\Delta]=\omega\norm\Delta^2$, which vanishes exactly when
$\Delta=0$.
\end{corollary}
\begin{proof}
By Theorem~\ref{thm:collapse_decomposition},
\begin{align*}
 DL(W)[-\Delta]&=-\langle\nabla_WL(W),\Delta\rangle
               =\kappa_1=\omega\norm\Delta^2,\\
 DL(W)[u]&=\frac{DL(W)[-\Delta]}{\norm\Delta}
          =\omega\norm\Delta\quad(\Delta\ne0),\\
 DL(W)[-\Delta]=0&\quad\Longleftrightarrow\quad\Delta=0.
 \qedhere
\end{align*}
\end{proof}

\subsection{The Score as a Ratio of First-Order Costs}
\label{sec:score_ratio}

\begin{proposition}[Matrix-level collapse-to-deletion ratio]
\label{prop:ratio_score}
Fix $a$ and $\omega>0$. Interpret $-\Delta_a$ and $-W_a$ as tuple
directions supported only on block $a$, and write
$N_a=DL(W)[-\Delta_a]$, $D_a=DL(W)[-W_a]$. These derivatives satisfy
\begin{equation}
 N_a=\omega\norm{\Delta_a}_F^2-\langle\rho_a,\Delta_a\rangle_F,
 \qquad D_a=\omega m_a-\langle\rho_a,W_a\rangle_F.
\end{equation}
If $\rho_a=0$, $N_a/D_a=S_a$; if $\rho=0$, consequently
\begin{equation}
 \bar S=\frac1{|\mathcal A|}\sum_a\frac{N_a}{D_a}.
 \label{eq:ratio_score}
\end{equation}
More generally, if $\varepsilon_a<\omega\norm{W_a}_F$, then $D_a>0$ and
\begin{equation}
 \left|\frac{N_a}{D_a}-S_a\right|
 \le\frac{2\varepsilon_a\sqrt{S_a}}{\omega\norm{W_a}_F-\varepsilon_a}.
 \label{eq:ratio_score_residual}
\end{equation}
\end{proposition}
\begin{proof}
All products below are on block $a$. For the residual bound, assume
$\varepsilon_a<\omega\norm{W_a}_F$. Orthogonality gives
\begin{align*}
 N_a&=-\langle\rho_a-\omega W_a,\Delta_a\rangle_F
     =\omega\norm{\Delta_a}_F^2-\langle\rho_a,\Delta_a\rangle_F,\\
 D_a&=-\langle\rho_a-\omega W_a,W_a\rangle_F
     =\omega m_a-\langle\rho_a,W_a\rangle_F,\\
 \rho_a=0&\quad\Longrightarrow\quad
 \frac{N_a}{D_a}=\frac{\omega\norm{\Delta_a}_F^2}{\omega m_a}=S_a,\\
 D_a&\ge\norm{W_a}_F(\omega\norm{W_a}_F-\varepsilon_a)>0,\\
 \frac{N_a}{D_a}-S_a
 &=\frac{S_a\langle\rho_a,W_a\rangle_F-
               \langle\rho_a,\Delta_a\rangle_F}{D_a},\\
 \left|\frac{N_a}{D_a}-S_a\right|
 &\le\frac{\varepsilon_a(S_a\norm{W_a}_F+\norm{\Delta_a}_F)}{D_a}\\
 &\le\frac{\varepsilon_a(S_a+\sqrt{S_a})}
             {\omega\norm{W_a}_F-\varepsilon_a}
  \le\frac{2\varepsilon_a\sqrt{S_a}}
             {\omega\norm{W_a}_F-\varepsilon_a}.
 \qedhere
\end{align*}
\end{proof}

\subsection{Finite Role-Collapse Cost}

\begin{theorem}[Score certificate for role-collapse cost]
\label{thm:certificate}
For every fixed labelling and $\omega\ge0$,
\begin{align}
 L(\bar W)-L(W)&\ge\omega\sum_aS_am_a+\bar D
                        -\varepsilon\norm\Delta-\rho_R,
 \label{eq:collapse_cost_lower_bound}\\
 L(\bar W)-L(W)&\ge\omega|\mathcal A|\bar S\min_a m_a+\bar D
                        -\varepsilon\norm\Delta-\rho_R.
 \label{eq:mean_score_certificate}
\end{align}
A strictly positive loss increase is certified when
$\varepsilon\norm\Delta+\rho_R<\omega\sum_aS_am_a+\bar D$.
\end{theorem}
\begin{proof}
The exact decomposition and the residual bounds give
\begin{align*}
 L(\bar W)-L(W)
 &=\omega\norm\Delta^2-\langle\rho,\Delta\rangle+\bar D+R\\
 &\ge\omega\norm\Delta^2+\bar D
       -\varepsilon\norm\Delta-\rho_R\\
 &=\omega\sum_aS_am_a+\bar D-\varepsilon\norm\Delta-\rho_R\\
 &\ge\omega|\mathcal A|\bar S\min_a m_a+\bar D
       -\varepsilon\norm\Delta-\rho_R.
 \qedhere
\end{align*}
\end{proof}

\begin{corollary}[Stationarity with zero logit remainder]
\label{cor:stationary_affine}
If $\rho=0$ and $r(x)=0$ almost surely, then
\begin{equation}
 L(\bar W)-L(W)=\omega\sum_aS_am_a+\bar D
 \ge\omega|\mathcal A|\bar S\min_a m_a.
 \label{eq:stationary_affine_exact}
\end{equation}
\end{corollary}
\begin{proof}
Substituting $R=0$ and $\rho=0$ into the exact decomposition yields
\begin{align*}
 L(\bar W)-L(W)&=\omega\norm\Delta^2+\bar D
               =\omega\sum_aS_am_a+\bar D\\
               &\ge\omega\sum_aS_am_a
                \ge\omega|\mathcal A|\bar S\min_a m_a.
 \qedhere
\end{align*}
\end{proof}

\subsection{Loss Advantage Requires Role-External Energy}

With unanalysed parameters fixed, let
\begin{equation}
 L_k^\star=\inf_{W':\;\#\operatorname{cols}_{\mathrm{distinct}}(W'_a)\le k\;\forall a}L(W').
 \label{eq:k_role_optimum}
\end{equation}
\begin{theorem}[Necessary energy for a loss advantage]
\label{thm:necessity}
Fix a labelling. If $L(W)\le L_k^\star-\delta$ with
$\delta>\rho_R$, then $\Lambda_1>0$ and
\begin{align}
 \norm\Delta=\left(\sum_aS_am_a\right)^{1/2}
 &\ge\frac{\delta-\rho_R}{2\Lambda_1},
 \label{eq:necessity}\\
 \bar S&\ge\frac{(\delta-\rho_R)^2}
                   {4|\mathcal A|\max_a m_a\,\Lambda_1^2}.
 \label{eq:mean_score_necessity}
\end{align}
\end{theorem}
\begin{proof}
Feasibility of the role model and Lemma~\ref{lem:bregman} give
\begin{align*}
 L(\bar W)&\ge L_k^\star\ge L(W)+\delta,\\
 \delta&\le L(\bar W)-L(W)
       \le\mathbb E\big|\ell_y(z_{\bar W})-\ell_y(z_W)\big|\\
       &\le2\mathbb E_x\norm{\delta z(x)}_\infty\\
       &\le2\mathbb E_x\norm{J_W(x)\Delta}_\infty
           +2\mathbb E_x\norm{r(x)}_\infty\\
       &\le2\Lambda_1\norm\Delta+\rho_R,\\
 0<\delta-\rho_R&\le2\Lambda_1\norm\Delta
       \quad\Longrightarrow\quad
       \Lambda_1>0,\quad\norm\Delta\ge\frac{\delta-\rho_R}{2\Lambda_1},\\
 |\mathcal A|\bar S\max_a m_a
       &\ge\norm\Delta^2\ge\frac{(\delta-\rho_R)^2}{4\Lambda_1^2}.
 \qedhere
\end{align*}
\end{proof}

\subsection{What the \dvl{} Labelling Determines}
\label{sec:score_labelling}

The following propositions characterize the normalized-readout clustering
objective and the role model's readouts.

\begin{proposition}[Normalized-readout clustering objective]
\label{prop:vocab_optimal}
Fix the trained readout map $A_a=\WU\bar J_a$ and set
$r_{a,i}=A_aw_{a,i}$ and
$\widetilde r_{a,i}=r_{a,i}/(\norm{r_{a,i}}_2+\eta)$, where $\eta>0$
is the numerical stabilizer denoted by $\varepsilon$ in the main text.
For any labelling $c$, let $\widetilde\mu_{a,c}$ be the mean of the
normalized readouts in that cluster and define
\begin{equation}
 J_a(c)=\sum_i\norm{\widetilde r_{a,i}-\widetilde\mu_{a,c_i}}_2^2.
 \label{eq:readout_objective}
\end{equation}
An ideal global minimizer $c_a^\star$ satisfies
\begin{equation}
 J_a(c_a^\star)=\min_{c:\;\text{at most }k\text{ labels}}J_a(c)
 \le J_a(c_a^{\mathrm{DVL}}).
 \label{eq:vocab_optimal}
\end{equation}
Lloyd assignment and centroid updates each weakly decrease the
normalized-readout objective.
\end{proposition}
\begin{proof}
Suppress $a$, and put
$F(c,\mu)=\sum_i\norm{\widetilde r_i-\mu_{c_i}}_2^2$.
With $\mu^t$ the current cluster means, the centroid identity and
nearest-centre assignment imply
\begin{align*}
 F(c,\mu)&=J(c)+\sum_{j:n_j>0}n_j\norm{\widetilde\mu_j-\mu_j}_2^2,
 \qquad J(c)=\min_\mu F(c,\mu),\\
 \min_c J(c)&=\min_\mu\sum_i\min_j\norm{\widetilde r_i-\mu_j}_2^2,\\
 F(c^{t+1},\mu^t)&\le F(c^t,\mu^t),\\
 J(c^{t+1})=F(c^{t+1},\mu^{t+1})
 &\le F(c^{t+1},\mu^t)\le F(c^t,\mu^t)=J(c^t).
 \qedhere
\end{align*}
\end{proof}

\begin{proposition}[Readouts under a fixed trained map]
\label{prop:readout_preserved}
For any labelling and the fixed linear map $A_a$, the role centroid
reproduces the cluster mean of the unnormalized readouts:
$A_a\bar w_{a,c}=\bar r_{a,c}:=n_{a,c}^{-1}\sum_{i\in C_{a,c}}r_{a,i}$.
The removed readout matrix is $A_a\Delta_a$ and its squared norm is
the within-cluster readout dispersion. If $A_a$ has full column rank,
\begin{equation}
 S_a\le\frac{\sum_i\norm{r_{a,i}-\bar r_{a,c_{a,i}}}_2^2}
                  {\sigma_{\min}(A_a)^2m_a}.
 \label{eq:readout_bound}
\end{equation}
\end{proposition}
\begin{proof}
Linearity and the smallest singular value give
\begin{align*}
 A_a\bar w_{a,c}
 &=\frac1{n_{a,c}}\sum_{i\in C_{a,c}}A_aw_{a,i}=\bar r_{a,c},\\
 A_a\bar W_a-A_aW_a&=-A_a\Delta_a,\\
 \norm{A_a\Delta_a}_F^2
 &=\sum_i\norm{r_{a,i}-\bar r_{a,c_{a,i}}}_2^2\\
 &\ge\sigma_{\min}(A_a)^2\sum_i\norm{\Delta_{a,i}}_2^2
  =\sigma_{\min}(A_a)^2m_aS_a.
 \qedhere
\end{align*}
\end{proof}

\begin{proposition}[Score bounds and permutation reference]
\label{prop:sandwich}
For a matrix with $n=n_a$ columns and $m=m_a>0$, define
\begin{equation}
 \begin{gathered}
 \bar w^{(1)}=\frac1n\sum_iw_i,\qquad
 m^{\mathrm c}=\sum_i\norm{w_i-\bar w^{(1)}}_2^2,\qquad
 S^{(1)}=m^{\mathrm c}/m,\\
 S^{\mathrm{opt}}(k)=\min_{c:\;\text{at most }k\text{ labels}}S_a(c).
 \end{gathered}
 \label{eq:one_cluster_score}
\end{equation}
For any labelling with $k_c\le k$ nonempty clusters,
$S^{\mathrm{opt}}(k)\le S_a(c)\le S^{(1)}$.
A raw-weight $k$-means run satisfies
$S^{\mathrm{opt}}(k)\le S^{\mathrm{run}}$.
For $n>1$ and a uniform random permutation $\pi$ of the labels,
\begin{equation}
 \mathbb E_\pi S_a(c\circ\pi)=\frac{n-k_c}{n-1}S^{(1)}.
 \label{eq:permutation_null}
\end{equation}
For $n=1$, all these scores are zero.
\end{proposition}
\begin{proof}
Write $\mathrm{SSE}(c)=mS_a(c)$ and $v_i=w_i-\bar w^{(1)}$, so that
$\sum_iv_i=0$. For a permuted cluster $C_j^\pi$ of fixed size $n_j$,
each column is included with probability $n_j/n$ and each distinct
pair with probability $n_j(n_j-1)/(n(n-1))$. Therefore, for $n>1$,
\begin{align*}
 mS^{\mathrm{opt}}(k)&\le\mathrm{SSE}(c)
 \le\sum_i\norm{w_i-\bar w^{(1)}}_2^2=m^{\mathrm c},\\
 \sum_{i\ne l}\langle v_i,v_l\rangle&=-\sum_i\norm{v_i}_2^2=-m^{\mathrm c},\\
 \mathbb E_\pi\left\|\sum_{i\in C_j^\pi}v_i\right\|_2^2
 &=\frac{n_j}{n}m^{\mathrm c}
   -\frac{n_j(n_j-1)}{n(n-1)}m^{\mathrm c}
  =\frac{n_j(n-n_j)}{n(n-1)}m^{\mathrm c},\\
 \mathbb E_\pi\mathrm{SSE}(c\circ\pi)
 &=m^{\mathrm c}-\sum_{j=1}^{k_c}\frac1{n_j}
                 \mathbb E_\pi\left\|\sum_{i\in C_j^\pi}v_i\right\|_2^2\\
 &=m^{\mathrm c}\left(1-\frac{\sum_j(n-n_j)}{n(n-1)}\right)
  =\frac{n-k_c}{n-1}m^{\mathrm c},\\
 \mathbb E_\pi S_a(c\circ\pi)&=\frac{n-k_c}{n-1}S^{(1)}.
 \qedhere
\end{align*}
\end{proof}

\section{More Discussions}
\label{app:more_discussions}

\subsubsection*{Q1. What Does a \dvl{} Readout Describe for One Parameter Column?}
The readout is the prompt-set average of the local logit response to a unit change in that column's activation coefficient (Eq.~\ref{eq:dvl}). For an individual prompt, the realized first-order contribution is $c_{\ell,m,i}(x)=f_{\ell,m,i}(x)r_{\ell,m,i}(x)$, as defined in Section~\ref{sec:dvl_readout}. Thus the signed vocabulary tail of a column identifies a response direction; the coefficient determines how that direction participates in a particular computation.

\subsubsection*{Q2. Does the Readout Predict Actual Changes to the Model Output?}

The finite-difference experiment directly compares the predicted local response with a small parameter perturbation. Across 720 cases from nine models, pooled coordinate-orientation agreement is $0.980$, while exact two-tail Top-32 overlap is $0.407$ and Recall@128 is $0.613$ (Table~\ref{tab:fd}). This distinction matters: the readout is more accurate about response direction and the broader affected vocabulary region than about the exact rank of every token. The ablation, steering, and swapping study then tests decoded behavior across 21 models (Section~\ref{sec:intervention}).

\subsubsection*{Q3. Could the Observed Structure Be a Generic Effect of Random Parameter Directions?}

The matched-null analysis compares concept-mean directions with Gaussian directions under the same model, depth, and module cells. Both excess kurtosis and top-1\% energy exceed the matched null on average in all 21 models, and the positive pattern remains after each model family is removed in turn (Appendix~\ref{app:matched_null_breadth}). These controls support the presence of structured vocabulary responses, while the finite-difference and intervention tests separately examine whether the responses track functional changes.

\subsubsection*{Q4. Do All Modules Provide Equally Precise Vocabulary Readouts?}

No. In Table~\ref{tab:fd}, orientation agreement ranges from $0.947$ for K to $0.997$ for FFN; exact Top-32 overlap ranges from $0.209$ for K to $0.565$ for FFN. The method uses the same local-response definition for Q, K, V, O, and FFN, but the empirical coordinate precision differs across these modules. The module-level values are therefore more informative than a single pooled score when interpreting a particular projection.

\section{Limitations}
\label{app:limitations}

This work studies parameter organization from the perspective of downstream vocabulary behavior, which provides only one view of the internal structure of language models. Learned representations are shaped by many interacting factors, including architecture, optimization, training data, and scale, and the relationships observed from a particular analytical perspective may not fully capture this broader complexity. More generally, connecting internal representational structure to high-level model behavior remains an open problem, especially across different model families and stages of training. Future work could explore how the proposed perspective interacts with complementary views of representation, training dynamics, and model scaling, toward a more unified understanding of how internal structure relates to emergent capabilities.

\end{document}